%% file: main.tex
\documentclass{article}

\usepackage{amsthm}
\usepackage[truedimen,margin=30mm]{geometry} 
\usepackage{hyperref}
\usepackage{url}
\usepackage{natbib}
\usepackage{booktabs}
\usepackage{lscape}
\usepackage{algorithm}
\usepackage{algorithmic}

\input{math_commands}

\input{original_math_commands}

\newcommand{\strtheorem}{Theorem}
\newcommand{\strassumption}{Assumption}
\newcommand{\strexample}{Example}
\newcommand{\strlemma}{Lemma}
\newcommand{\strproposition}{Proposition}
\newcommand{\strbackground}{Background Knowledge}
\newcommand{\strremark}{Requirement}
\newcommand{\strdefinition}{Definition}

\newtheorem{theorem}{\strtheorem}
\newtheorem{assumption}{\strassumption}

\newtheorem{lemma}{\strlemma}

\newtheorem{remark}{\strremark}

\usepackage{cleveref}
\crefname{theorem}{Theorem}{Theorems}
\crefname{assumption}{Assumption}{Assumptions}
\crefname{example}{Example}{Examples}
\crefname{lemma}{Lemma}{Lemmas}
\crefname{background}{Background Knowledge}{Background Knowledge}
\crefname{remark}{Requirement}{Requirements}
\crefname{definition}{Definition}{Definitions}
\crefname{algorithm}{Algorithm}{Algorithms}
\crefname{figure}{Fig}{Figs}
\crefname{table}{Table}{Tables}

\usepackage{graphicx}
\graphicspath{ {./imgs/} }

\allowdisplaybreaks[1]
\title{BODAME: Bilevel Optimization for Defense Against Model Extraction}
\author{Yuto Mori$^{1\dag}$, Atsushi Nitanda$^{1, 2, 3\ddag}$, Akiko Takeda$^{1, 2\star}$\\
\vspace{2mm}\\
\normalsize{\textit{$^1$Graduate School of Information Science and Technology, The University of Tokyo, Japan}}\\
\normalsize{\textit{$^2$Center for Advanced Intelligence Project, RIKEN, Japan}}\\
\normalsize{\textit{$^3$PRESTO, Japan Science and Technology Agency, Japan}}\\
\small{Email: $^\dag$yuto\_mori@mist.i.u-tokyo.ac.jp, $^\ddag$nitanda@mist.i.u-tokyo.ac.jp, $^\star$takeda@mist.i.u-tokyo.ac.jp}
}

\date{}
\begin{document}
\maketitle

\input{contents/abstract}
\input{contents/introduction}
\input{contents/formulation}
\input{contents/kernel_bodame}
\input{contents/sgd_bodame}

\input{contents/kernel_experiments}
\input{contents/sgd_experiments}

\input{contents/conclusion}
\input{contents/acknowledgements}

\bibliographystyle{apalike}

\input{main-bbl.bbl}
\clearpage

\setcounter{section}{0}
\setcounter{subsection}{0}

\input{contents/appendix}

\end{document}

%% file: math_commands.tex
\usepackage{amsmath,amsfonts,bm}

\def\eqref#1{equation~\ref{#1}}

\def\1{\bm{1}}

\DeclareMathAlphabet{\mathsfit}{\encodingdefault}{\sfdefault}{m}{sl}
\SetMathAlphabet{\mathsfit}{bold}{\encodingdefault}{\sfdefault}{bx}{n}

\DeclareMathOperator*{\argmin}{arg\,min}

%% file: original_math_commands.tex
\DeclareMathOperator*{\E}{\mathbb{E}}
\DeclareMathOperator*{\supp}{\mathrm{supp}}
\DeclareMathOperator*{\Ker}{\mathrm{Ker}}

\newcommand{\obj}{\mathrm{(o)}}
\newcommand{\attacker}{\mathrm{(a)}}
\newcommand{\const}{\mathrm{(c)}}
\newcommand{\train}{\mathrm{(tr)}}
\newcommand{\f}{f}
\newcommand{\g}{g}
\newcommand{\h}{h}
\newcommand{\gparam}{\theta}
\newcommand{\gParam}{\Theta}
\newcommand{\hparam}{w}
\newcommand{\hParam}{W}
\newcommand{\fSpace}{F}
\newcommand{\gSpace}{G}
\newcommand{\hSpace}{H}
\newcommand{\Spacein}{\mathcal{X}}
\newcommand{\Spaceout}{\mathcal{Y}}
\newcommand{\queryin}{x}
\newcommand{\queryout}{y}
\newcommand{\rvin}{X}

\newcommand{\idxobj}{j}
\newcommand{\idxattacker}{i}
\newcommand{\idxconst}{j'}
\newcommand{\idxtrain}{s}

\newcommand{\nobj}{m}
\newcommand{\nattacker}{n}
\newcommand{\nconst}{m'}
\newcommand{\ntrain}{M}

\newcommand{\epsquality}{\varepsilon}

\newcommand{\Pdist}{P}
\newcommand{\q}{q}
\newcommand{\Qdist}{Q}
\newcommand{\kernel}{k}
\newcommand{\lambdaattacker}{\lambda}
\newcommand{\A}{A}
\newcommand{\veca}{a}
\newcommand{\B}{B}
\newcommand{\vecb}{b}
\newcommand{\C}{C}
\newcommand{\Set}{S}

\newcommand{\eigenvec}{z}
\newcommand{\eigenvalue}{\alpha^{*}}
\newcommand{\eigenkkt}{\alpha}
\newcommand{\eigenspace}{Z}
\newcommand{\stepinner}{t}
\newcommand{\stepouter}{u}
\newcommand{\maxstepinner}{T}
\newcommand{\maxstepouter}{U}
\newcommand{\lrinner}{\eta}

\newcommand{\lrouter}{\beta}

\newcommand{\update}{\Phi}
\newcommand{\Z}{Z}
\newcommand{\loss}{l}
\newcommand{\ind}{\bm{1}}
\newcommand{\lipconstant}{L}
\newcommand{\distmu}{d_{\mu}}

\newcommand{\Hilbert}{\mathcal{H}}

\newcommand{\Hardmat}{H}
\newcommand{\nulldim}{d}
\newcommand{\nulleigenvec}{v}

\newcommand{\lobj}{\loss^{\obj}}
\newcommand{\lattacker}{\loss^{\attacker}}
\newcommand{\lconst}{\loss^{\const}}
\newcommand{\wargmin}{\tilde{\hparam}(\gparam)}
\newcommand{\hargmin}{\h_{\wargmin}}

\newcommand{\xobjsample}{\queryin^{\obj}_{\idxobj}}
\newcommand{\xattackersample}{\queryin^{\attacker}_{\idxattacker}}
\newcommand{\xconstsample}{\queryin^{\const}_{\idxconst}}
\newcommand{\xtrainsample}{\queryin^{\train}_{\idxtrain}}
\newcommand{\kattacker}{\kernel_{\h}}
\newcommand{\ktrue}{\kernel^{*}}
\newcommand{\gammaa}{\gamma_{\veca}}
\newcommand{\gammab}{\gamma_{\vecb}}
\newcommand{\minibatchobj}{D^{\obj}}
\newcommand{\minibatchattacker}{D^{\attacker}}

\newcommand{\regconst}{\lambda^{\const}}
\newcommand{\Lobj}{L^{\obj}}
\newcommand{\Lconst}{L^{\const}}
\newcommand{\dhParam}{d'}
\newcommand{\dgParam}{d}
\newcommand{\ftrue}{\f_{\gparam^{*}}}

\newcommand{\hundefended}{\h_{\tilde{\hparam}(\gparam^{*})}}
\newcommand{\hdefended}{\h_{\tilde{\hparam}(\gparam_{\mathrm{opt}})}}

%% file: contents/abstract.tex
\begin{abstract}
Model extraction attacks have become serious issues for service providers using machine learning.
We consider an adversarial setting to prevent model extraction under the assumption that attackers will make their best guess on the service provider's model using query accesses, and propose to build a surrogate model that significantly keeps away the predictions of the attacker's model from those of the true model.
We formulate the problem as a non--convex constrained bilevel optimization problem and show that for kernel models, it can be transformed into a non--convex 1--quadratically constrained quadratic program with a polynomial--time algorithm to find the global optimum.
Moreover, we give a tractable transformation and an algorithm for more complicated models that are learned by using stochastic gradient descent--based algorithms.
Numerical experiments show that the surrogate model performs well compared with existing defense models when the difference between the attacker's and service provider's distributions is large.
We also empirically confirm the generalization ability of the surrogate model.
\end{abstract}

%% file: contents/introduction.tex
\section{Introduction}
Model extraction attacks \citep{lowd2005adversarial, tramer2016stealing} have become serious problems as more and more services based on machine learning are released as application programming interfaces (APIs).
A service provider who uses such an API is at risk of having its service imitated by model extraction attacks in addition to other types of attacks using adversarial examples \citep{szegedy2014intriguing, goodfellow2015explaining} or model inversion \citep{fredrikson2015model}, that become easier to execute because of model extraction attacks.
If an attacker is able to steal the internal model, it might become public and people might stop using the service because they can use a copy in their local environments.
Furthermore, there is even a risk that a similar service will be developed by an attacker.
Model extraction attacks thus pose huge risks for service providers who need to defend their machine learning models against attackers.

Although model extraction attacks have been recently studied from various viewpoints, focusing on the targeted model classes \citep{tramer2016stealing, bastani2017interpretability, oh2018towards, milli2019model, rolnick2020reverse}, the method to choose queries \citep{orekondy2019knockoff, chandrasekaran2020exploring}, the attacker's objectives \citep{jagielski2020high_accuracy}, and the attacker's knowledge about the defender \citep{batina2019csi, pal2020activethief, krishna2020thieves}, there is not much research aimed at defending against such attacks, except those summarized in \cref{tab:defense-methods}.
These defense methods in \cref{tab:defense-methods} consider preserving the quality of the output or the surrogate function not to be far from the defender's true model, which is characteristic of the model extraction problems.
For example, rounding \citep{tramer2016stealing}, reverse sigmoid \citep{lee2019defending}, boundary differentially private layer (BDPL) \citep{zheng2019bdpl}, and maximum angular deviation (MAD) \citep{orekondy2020prediction} are the defending methods that change an output of the defender's true model by, e.g. adding a noise.
There is another type of a defense against model extraction; \citep{alabdulmohsin2014adding} builds a surrogate function that is robust against exploratory attacks, and their approach is theoretically analyzed by using active learning in terms of randomization \citep{chandrasekaran2020exploring}.

\begin{table}[ht]
\caption{\textbf{Existing work on defenses against model extraction.} \textbf{Objective} means the goal for the defender. Disagreement means to make an attacker's model apart from the defender's true model. Detection means to detect whether model extraction can be performed. \textbf{S/P} means surrogate/perturbation: it shows whether a method uses a surrogate function or a perturbation as the defense. \textbf{OptDef} means whether an optimal defense is taken to prevent the model extraction by an attacker. \textbf{Knowledge} is the defender's knowledge about an attacker that is needed for the procedure to make a defense. \textbf{B/O} means batch/online: it shows whether a defender can get attacker's queries in batch or online. \textbf{C/R} means classification/regression: it shows whether a method can be applied to a classification or regression task. \textbf{Model} is the defender's model that is the target of each method. Linear--SVM means support vector machine for linear models. One--hot Classifier means a model of which output is one--hot vector. KR means kernel regression.}

\label{tab:defense-methods}
\begin{center}
\fontsize{6.5pt}{0.25cm}\selectfont
\begin{tabular}{lccccccc} 
    \toprule
                  & Objective & S/P & OptDef & Knowledge & B/O & C/R & Model\\
    \midrule
    \citep{alabdulmohsin2014adding} & Disagreement & S & No & - & - & C & Linear--SVM\\
    Rounding \citep{tramer2016stealing} & Disagreement & P & No & Query & O & C/R & - \\
    \citep{kesarwani2018model} & Detection & - & No & Query, Model & B & C & Decision Tree \\
    Reverse Sigmoid \citep{lee2019defending} & Disagreement & P & No & Query & O & C & One--hot Classifier \\
    PRADA \citep{juuti2019prada} & Detection & - & No & Query & B & C/R & -\\
    BDPL \citep{zheng2019bdpl} & Disagreement & P & No & Query & O & C & -\\
    MAD \citep{orekondy2020prediction} & Disagreement & P & Yes & Query, Gradient & O & C & One--hot Classifier\\[2pt] \hline \addlinespace[2pt]
    \textbf{BODAME--KRR/KR (Ours.)} & \textbf{Disagreement} & \textbf{S} & \textbf{Yes} & \textbf{Query, Model} & \textbf{B} & \textbf{R} & \textbf{KR}\\ 
    \textbf{BODAME--SGD/SGA (Ours.)} & \textbf{Disagreement} & \textbf{S} & \textbf{Yes} & \textbf{Query, Model} & \textbf{B} & \textbf{C/R} & \textbf{Differentiable}\\
    \bottomrule
\end{tabular}
\end{center}
\end{table}

\begin{figure}
    \centering
    \includegraphics[width=10cm]{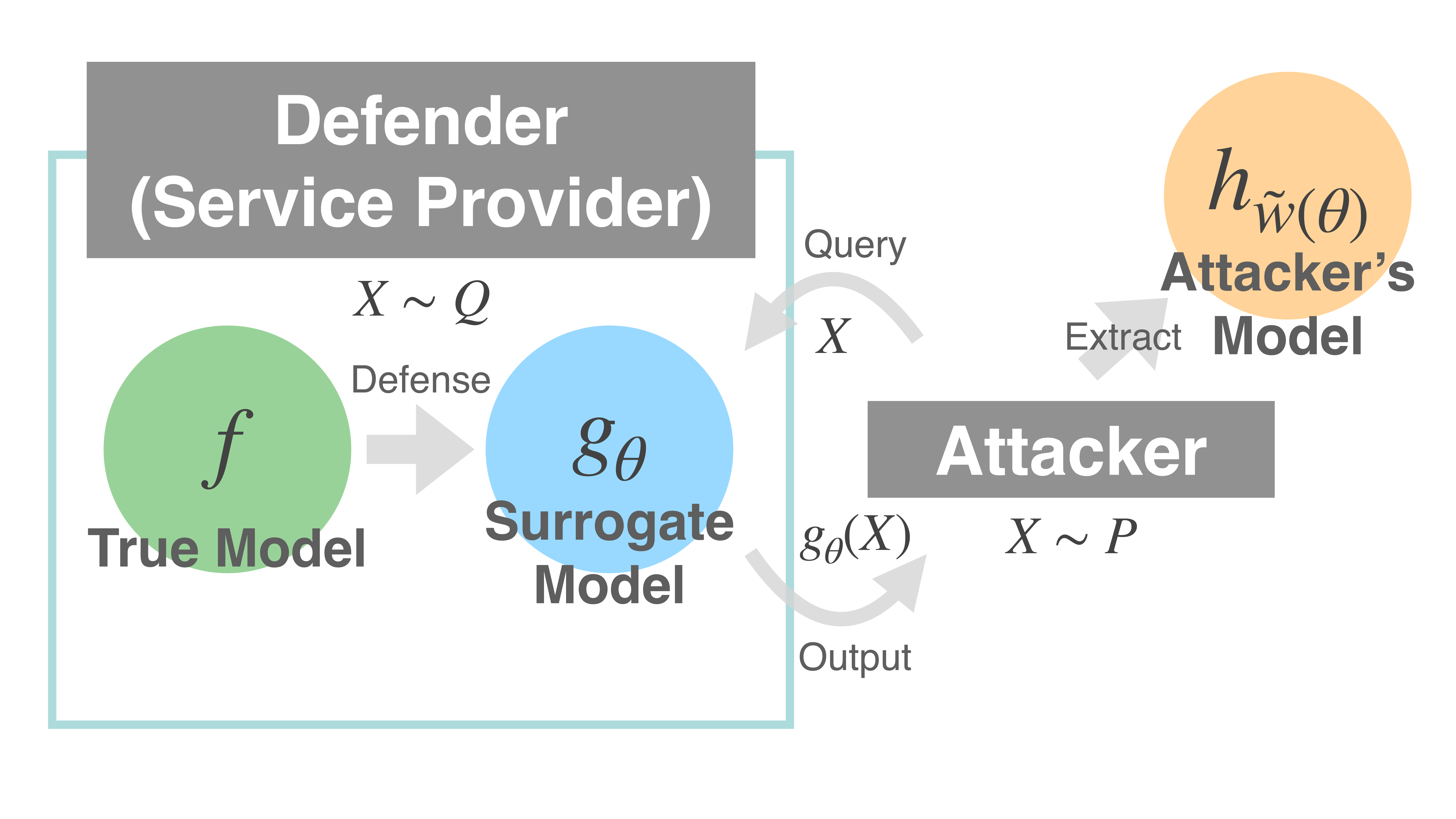}
    \caption{Overview of our attacker--defender framework}
    \label{overview}
\end{figure}

In this paper, we consider an adversarial setting to prevent model extraction by attackers, unlike the above--mentioned defense methods, under the assumption that attackers will make their best guess on the defender's model by query accesses.
The purpose of our model is to build surrogate models for defenders in such a way that the predictions between the defender's true model and the attackers' models are significantly different under the defender's distribution.
By utilizing the surrogate models instead of their true models for API services, service providers will not be damaged even if their surrogate models are imitated from the outputs of their surrogate models (see \cref{overview} for the structure of our model).
We formulate the problem as a bilevel optimization problem to construct a surrogate model including the attacker’s problem that imitates the defender’s model as a constraint. 
Another constraint requires the surrogate model to be close to the true defender's model.

The resulting constrained bilevel optimization problem is a non--convex problem.
We show that the problem can be transformed into a non--convex quadratic optimization problem having one quadratic constraint (1--QCQP) with a polynomial–time algorithm to find the global optimum in the setting that the attacker uses kernel ridge regression and the defender uses a kernel regression model.
We also devise an algorithm to find solutions to more general machine learning models under the assumption that attackers use a stochastic gradient descent (SGD)--based algorithm to imitate the defender's models.
We performed numerical experiments confirming that our algorithms succeed in the defending the model when the distributions between the attackers and defenders are substantially different.
We also show that as long as the attacker's queries follow the same distribution of queries used for defense, the defenders can use the same surrogate models without additional updates.

%% file: contents/formulation.tex
\section{Bilevel Optimization for Defense Against Model Extraction}
\label{sec:formulation-analysis}
In this section, we formulate Bilevel Optimization for Defense Against Model Extraction ({\bf BODAME}), theoretically analyze two extreme cases, and show the BODAME formulation in an empirical setting on the attacker's samples and defender's samples.

\subsection{Attacker--Defender setting}
In this paper, we will consider the following setting.

For the attacker, we assume {\em fidelity extraction} \citep{jagielski2020high_accuracy};
the objective of the attacker is to find a parameter of the attacker model that minimizes an expected loss function on the attacker's input distribution by using the outputs of the defender's surrogate model as supervised data. 
Here, we incorporate the attacker's optimization problem as a constraint, which leads to a bilevel formulation.
The optimization problem in the fidelity extraction has the same objective function as supervised learning models;
therefore, we can apply the standard algorithms developed for supervised learning to the attacker's optimization problem.

The defender is assumed to know some information about attackers' models, in particular, the loss functions, hyperparameters, and solution methods (or update formula) of the attacker's models including some of the initial parameter values. Although these settings are stronger than those of prior work shown as \cref{tab:defense-methods}, they seem inevitable in the adversarial setting.
In particular, examining this setting will shed light on the limits of the defense against model extraction.

The goal of the defender is to build a surrogate model that maximizes the differences between predictions given the defender's distribution between the defender's true model and the attacker's model.
At the same time, the defender wants to keep almost the same prediction quality between the surrogate model and the true defender's model.
These two objectives lead us to formulate a constrained version of bilevel optimization.

\subsection{Expected BODAME}
We will denote the input space as $\Spacein$, the output space as $\Spaceout$, and the defender and attacker distribution as $\Qdist$ and
$\Pdist$, respectively, on $\Spacein$.
We set function spaces $\fSpace$, $\gSpace$, and $\hSpace$ as model classes.
In particular, $\gSpace$ and $\hSpace$ are parameterized function spaces with parameter spaces $\gParam$ and $\hParam$, respectively, that is, $\gSpace = \{\g_{\theta}: \Spacein \to \Spaceout \mid \gparam \in \gParam\}$ and $\hSpace = \{\h_{\hparam}: \Spacein \to \Spaceout \mid \hparam \in \hParam\}$.
The defender has the true model $\f \in \fSpace:\Spacein \to \Spaceout$.
Because a naive deployment of the true model has some risk of the true model being stolen, the defender prepares a parameterized surrogate model $\g_{\gparam} \in \gSpace$ so as to be $\epsquality$--approximation ($\epsquality \ge 0$) to the true model $\f$ on a loss function $\lconst: \Spaceout \times \Spaceout \to \mathbb{R}$ on the defender's distribution $\Qdist$.
The attacker builds a parameterized copy--model $\h_{\hparam} \in \hSpace$ of the true model $\g_{\gparam}$ as a result of minimizing the attacker's loss function $\lattacker: \Spaceout \times \Spaceout \to \mathbb{R}$ on the attacker's input distribution $\Pdist$ as (\ref{expected-inner-minimization}).
The defender wants to prevent the attacker from re--training by maximizing its objective loss function $\lobj: \Spaceout \times \Spaceout \to \mathbb{R}$ between $\f$ and $\h_{\hparam}$ on $\Qdist$ as (\ref{expected-objective}).
Accordingly, we obtain the following formulation:
\begin{align}
    \max_{\gparam \in \gParam} \quad & \E_{\rvin \sim \Qdist} [\lobj(\f(\rvin), \hargmin(\rvin))], \label{expected-objective} \\
    \mathrm{s.t.} \quad & \wargmin = \argmin_{\hparam \in \hParam} \E_{\rvin \sim \Pdist} [\lattacker(\g_{\gparam}(\rvin), \h_{\hparam}(\rvin))], \label{expected-inner-minimization}  \\
    \quad &  \E_{\rvin \sim \Qdist} [\lconst(\f(\rvin), \g_{\gparam}(\rvin)) ]  \le \epsquality. \label{expected-constraint-quality}
\end{align}
Condition (\ref{expected-constraint-quality}) originates from the motivation that the defender does not want to make the surrogate model worse in order to guarantee the quality of the service for users who possess benign data following the true distribution $\Qdist$.
We call this formulation the expected Bilevel Optimization for Defense Against Model Extraction (BODAME) framework.
The resulting attacker model is denoted as $\hargmin$ and the surrogate model is denoted as $\g_{\gparam}$, where $\gparam$ is a solution of the problem (\ref{expected-objective})--(\ref{expected-constraint-quality}).

\subsection{Two extreme cases for expected BODAME}
We are now interested in this question: ``What are the circumstances in which the defender succeeds or fails in defending against model extraction?'', which can be rephrased as ``How large or small can the optimal value of expected BODAME be?''
The functions $\g$ and $\h$ need to be parameterized as in (\ref{expected-objective}), (\ref{expected-inner-minimization}), and (\ref{expected-constraint-quality}), but we will omit the notation in this subsection in order to focus on properties of the model as functions.
Here, we call $\lattacker$ a {\em p-distance} function in this subsection if it satisfies partial conditions of distance functions that wherein for all $\queryout, \queryout' \in \Spaceout, \lattacker(\queryout, \queryout') \ge 0$ and $\queryout = \queryout'$ if and only if $ \lattacker(\queryout, \queryout') = 0$.
First, we show a sufficient condition for a defender not to make an adequate defense.
\begin{theorem}[Sufficient condition for almost $0$ optimal value]
    Let $\lattacker$ be a p-distance function and $\lobj = \lconst$, and also suppose that $\supp(\Qdist) = \supp(\Pdist)$ for distributions and $\gSpace = \hSpace$ for function spaces. Then, the optimal value (\ref{expected-objective}) is upper bounded by $\epsquality$ for any $\f \in \fSpace$. \label{thm:trivial-case}
\end{theorem}
We give a proof of \cref{thm:trivial-case} in the supplementary material.
\cref{thm:trivial-case} implies that since the attacker can imitate the defender's model almost surely on $\Pdist$ by using a p-distance loss function such as a squared loss function and the attacker has the distribution in the same region with defender's one (i.e., $\supp(\Qdist) = \supp(\Pdist)$), the defender fails to defend against attacks because of the constraint (\ref{expected-constraint-quality}).
Next, we show an example where a defender can defend against an attacker in a regression setting.
\begin{theorem}[Sufficient condition for infinitely large optimal value]
   Let $\lattacker$ be a p-distance function and 
   suppose that $\Spacein = \Spaceout = \mathbb{R}$,
   $\lobj(\queryout, \queryout') = \lconst(\queryout, \queryout') = (\queryout - \queryout')^{2}$, $\Qdist$ and $\Pdist$ are absolutely continuous measures on $\mathbb{R}$, $\supp(\Qdist) = [a, b], \supp(\Pdist) = [c, d], (a < b \le c < d)$, $\fSpace$ and $\hSpace$ are $\lipconstant$-Lipschitz function spaces and $\gSpace = \{\g: \mathbb{R} \to \mathbb{R}\}$.
   Then, the optimal value of (\ref{expected-objective}) becomes $\infty$.
   \label{thm:infty-case}
\end{theorem}
We give a proof of \cref{thm:infty-case} in the supplementary material.
\cref{thm:infty-case} implies that the defender can mount a defense by making a surrogate model which is largely different from the true model only in $\Pdist$; therefore, the difference between the supports of $\Pdist$ and $\Qdist$ is a crucial condition.
We can easily extend \cref{thm:infty-case} to the multi--dimensional input space.

\subsection{Empirical batch BODAME}
\label{subsec:empirical-batch-bodame}

For practical purposes, we will consider an empirical setting.
Suppose that we have i.i.d. samples of the attacker $\{\xattackersample\}_{\idxattacker=1}^{\nattacker}$ ($\xattackersample \sim \Pdist$, $\nattacker \in \mathbb{N})$, i.i.d. samples of the defender used in the constraint $\{\xconstsample\}_{\idxconst=1}^{\nconst}$ ($\xconstsample \sim \Qdist$, $\nconst \in \mathbb{N}$), and i.i.d. samples of the defender used in the objective function $\{\xobjsample\}_{\idxobj=1}^{\nobj}$ ($\xobjsample \sim \Qdist$, $\nobj \in \mathbb{N}$) instead of the attacker's and defender's distributions.
The corresponding empirical batch BODAME is
\begin{align}
 \hspace{-2mm}   \max_{\gparam \in \gParam} \quad & \hspace{-1mm} \frac{1}{\nobj} \sum_{\idxobj=1}^{\nobj} \lobj(\f(\xobjsample), \hargmin(\xobjsample)),  \label{objective}\\
 \hspace{-2mm}   \mathrm{s.t.}  \quad & \hspace{-1mm}\wargmin = \argmin_{\hparam
    \in \hParam} \frac{1}{\nattacker} \sum_{\idxattacker=1}^{\nattacker} \lattacker(\g_{\gparam}(\xattackersample), \h_{\hparam}(\xattackersample)), \label{inner_minimization} \\
\hspace{-2mm}    \quad & \hspace{-1mm}\frac{1}{\nconst} \sum_{\idxconst=1}^{\nconst} \lconst(\f(\xconstsample), \g_{\gparam}(\xconstsample))  \le \epsquality. \label{constraint_quality}
\end{align}
The resulting optimization problem is difficult to solve; it is a non--convex problem due to (\ref{objective}) and an optimization problem (\ref{inner_minimization}) is included as a constraint. 
In the next two sections, we provide methods for solving the problem.

%% file: contents/kernel_bodame.tex
\section{BODAME--KRR/KR}
\label{sec:kernel-bodame}
We consider empirical batch BODAME, where we assume that the defender uses a kernel regression model and the attacker uses a kernel ridge regression model.
Although the optimization problem is still non--convex, we propose a polynomial--time algorithm with an easy implementation to find the global optimum for the non--convex problem.

\subsection{1--QCQP formulation}
\label{subsec:kernel-formulation}
On the basis of the representer theorem (see e.g. Theorem 4.2. in \citep{scholkopf2002learning}), we assume that the true model $\f$ can be expressed as $\f(\queryin) = \sum_{\idxtrain=1}^{\ntrain} \gparam^{*}_{\idxtrain} \ktrue(\queryin, \xtrainsample)$ where $\gparam^{*} \in \mathbb{R}^{\ntrain}$, $\ktrue : \Spacein \times \Spacein \to \mathbb{R}$ is a positive--definite kernel, $\ntrain \in \mathbb{N}$, $\{\xtrainsample\}_{\idxtrain=1}^{\ntrain}$ are training samples which are identically and independently distributed in $\Qdist$.
Moreover, we assume a surrogate model $\g_{\gparam} (\queryin) = \sum_{\idxtrain=1}^{\ntrain} \gparam_{\idxtrain} \ktrue(\queryin, \xtrainsample)$ ($\gparam \in \mathbb{R}^{\ntrain}$), which means that the surrogate model uses the same training data and kernel as the true model.
Further, we assume that the attacker's model is a kernel ridge regression model, but the attacker does not precisely know the defender's kernel in general.
Therefore, the attacker's model $\h_{\hparam}$ can be expressed as $\sum_{\idxattacker=1}^{\nattacker} \hparam_{\idxattacker} \kattacker(\queryin, \xattackersample)$, where $\hparam \in \mathbb{R}^{\nattacker}$ and $\kattacker$ is also a positive--definite kernel different in general from the defender's kernel in general.
The attacker's loss function is $\lattacker(\g_{\gparam}(\queryin), \h_{\hparam} (\queryin)) = ( \g_{\gparam}(\queryin) - \h_{\hparam} (\queryin))^{2} + \lambdaattacker \|\h_{\hparam}\|^2_{\Hilbert}$ where $\lambdaattacker > 0$ and $\|\cdot\|_{\Hilbert}$ denotes the norm induced from the reproducing kernel Hilbert space $\Hilbert$ corresponding to $\kattacker$.
Here, we assume that the defender knows some of the attacker's information, as follows:
\begin{remark}[Defender's knowledge about the attacker's kernel model]
  The defender knows that the definitions of the functions $\lattacker$ and $\h_{\hparam}$, including the hyperparameter value $\lambda$
  and  the attacker's samples $\{\xattackersample\}_{\idxattacker=1}^{\nattacker}$.
  \label{kernel-attacker-assumption}
\end{remark}

Based on \cref{kernel-attacker-assumption}, lower--level optimization (\ref{inner_minimization}) is tractable with an analytic solution:
\begin{align}
    \wargmin = (K_{1} + \lambdaattacker I_{\nattacker})^{-1} K_{2} \gparam, \label{w_kernel_ridge}
\end{align}
where $K_{1} = (\kattacker(\xattackersample, \queryin_{\idxattacker'}^{\attacker}))_{\idxattacker=1, \idxattacker'=1}^{\nattacker, \nattacker} \in \mathbb{R}^{\nattacker \times \nattacker}$ and $K_{2} = (\ktrue(\xattackersample, \xtrainsample))_{\idxattacker=1, \idxtrain=1}^{\nattacker, \ntrain} \in \mathbb{R}^{\nattacker \times \ntrain}$.
We should also note that $K_{3} = (\kattacker(\xobjsample, \xattackersample))_{\idxobj=1, \idxattacker=1}^{\nobj, \nattacker} \in \mathbb{R}^{\nobj \times \nattacker}$, $K_{4} = (\ktrue(\xconstsample, \xtrainsample))_{\idxconst=1, \idxtrain=1}^{\nconst, \ntrain} \in \mathbb{R}^{\nconst \times \ntrain}$, $\f_{1} = (\f(\xobjsample))_{\idxobj=1}^{\nobj} \in \mathbb{R}^{\nobj}$ and $\f_{2} = (\f(\xconstsample))_{\idxconst=1}^{\nconst} \in \mathbb{R}^{\nconst}$.
We use the following matrix and vector notations:
\begin{align*}
&\tilde{\A} = K_{3}  (K_{1} + \lambdaattacker I_{\nattacker})^{-1} K_{2}, \ 
\A = (1/\nobj)  \tilde{\A}^{\top}  \tilde{\A},\\
&\veca = (1/\nobj) \tilde{\A}^{\top} \f_{1}, \gammaa = (1/\nobj)  \f_{1}^{\top} \f_{1}, \ \B = (1/\nconst) K_{4}^{\top} K_{4}, \\
&\vecb = (1/\nconst) K_{4}^{\top} \f_{2}, \ 
\gammab = (1/\nconst) \f_{2}^{\top} \f_{2}.
\end{align*}
After putting (\ref{w_kernel_ridge}) into (\ref{objective}) and assuming that $\lobj$ and $\lconst$ are both squared loss functions, we can rewrite empirical batch BODAME
(\ref{objective})--(\ref{constraint_quality}) as follows:
\begin{align}
    \max_{\gparam \in \mathbb{R}^{\ntrain}} \quad & \gparam^{\top} \A \gparam - 2 \veca^{\top} \gparam + \gammaa, \label{kernel-objective}\\
    \mathrm{s.t.} \quad & \gparam^{\top} \B \gparam - 2 \vecb^{\top} \gparam + \gammab \le \epsquality. \label{kernel-constraint}
\end{align}
We can easily check that $\A$ and $\B$ are positive semi--definite and thereby that this problem is a non--convex 1--quadratically constrained quadratic program (1--QCQP).
We know a feasible solution exists because the parameter of the true model $\gparam^{*}$ always satisfies (\ref{kernel-constraint}).
We call this optimization problem \textbf{BODAME--KRR/KR} since the attacker uses a kernel ridge regression (KRR) model and the defender uses a kernel regression (KR) model.

\subsection{Algorithm for BODAME--KRR/KR}
We use an efficient algorithm \citep{adachi2017solving} with an easy implementation based on solving a generalized eigenvalue problem. To make the algorithm applicable to problem (\ref{kernel-objective}) and (\ref{kernel-constraint}), we use the following assumption:
\begin{assumption}
$\B$ is positive--definite. \label{B-PDS-assumption}
\end{assumption}
Under Assumption \ref{B-PDS-assumption}, we obtain the following formulation by changing a parameter from $\gparam$ to $\hat{\gparam} = \gparam - \B^{-1} \vecb$:
\begin{align}
    \min_{\|\hat{\gparam}\|_{\B} \le \hat{\epsquality}} \quad \frac{1}{2} \hat{\gparam}^{\top} \hat{\A} \hat{\gparam} +  \hat{\veca}^{\top} \hat{\gparam}, \label{acutual-problem}
\end{align}
where $\| \hat{\gparam} \|_{\B} = \sqrt{\hat{\gparam}^{\top} \B \hat{\gparam}}$, $ \hat{\A} = - 2 \A$, $\hat{\veca} = 2 \veca - 2 \A \B^{-1} \vecb$,  and $\hat{\epsquality} = \sqrt{\epsquality + \vecb^{\top} \B^{-1} \vecb - \gammab}$.

Since the problem (\ref{acutual-problem}) is a concave function minimization on a nonempty compact convex set, it is ensured that the optimal solution exists at the boundary of the inequality constraint (see, e.g., Theorem 1.1, \citep{horst1996global}). We can reduce some of the procedures in the original algorithm by exploiting the concave property to obtain the following \cref{global_opt_kernel}. Suppose that $\eigenspace$ is the eigenspace of $\A$ and $- \B$ corresponding to the rightmost generalized eigenvalue.
\begin{theorem}[Global optimization for BODAME--KRR/KR where $\hat{\veca} \not\perp \eigenspace$]
On the basis of \cref{kernel-attacker-assumption}, if \cref{B-PDS-assumption} and $\hat{\veca} \not\perp \eigenspace$ hold, then the defender can find the global optimum of the problem (\ref{kernel-objective}) and (\ref{kernel-constraint}), which is equivalent to (\ref{acutual-problem}) by applying \cref{kernel_bodame_algo}. \label{global_opt_kernel}
\end{theorem}
We give the proof of \cref{global_opt_kernel} in the supplementary material.
We assume $\hat{\veca} \not\perp \eigenspace$ in \cref{global_opt_kernel} which leads to a simplified algorithm, but this assumption is not necessary if we use Algorithm 2 in \citep{adachi2017solving} (See the supplementary material).

\begin{algorithm}
\caption{Algorithm for BODAME--KRR/KR by solving generalized eigenvalue problems}
\label{kernel_bodame_algo}
\begin{algorithmic}[1]
    \REQUIRE $\A$, $\veca$, $\B$, $\epsquality$, $\vecb$, $\gammab$
    \ENSURE Surrogate parameter $\gparam_{\mathrm{opt}}$
        \STATE Compute $\B^{-1} \vecb, \hat{\A}, \hat{\veca}, \hat{\epsquality}$
        \STATE Compute the smallest $\eigenvalue \in \mathbb{R}$ and (divided) eigenvectors $\eigenvec_1, \eigenvec_2 \in \mathbb{R}^\ntrain$ of following problem:
        \begin{align}
            \begin{pmatrix} - \B & \hat{\A} \\ \hat{\A} & - \frac{\hat{\veca}\hat{\veca}^{\top}}{\hat{\epsquality}^2}\end{pmatrix} \begin{pmatrix} \eigenvec_1 \\ \eigenvec_2 \end{pmatrix} = \eigenvalue \begin{pmatrix} O & \B \\ \B & O \end{pmatrix} \begin{pmatrix} \eigenvec_1 \\ \eigenvec_2\end{pmatrix} \label{kernel_eigen_prob}
        \end{align} 
        \STATE $\hat{\gparam}_{\mathrm{opt}} = - \mathrm{sgn} (\hat{\veca}^{\top} \eigenvec_2) \hat{\epsquality} \frac{\eigenvec_1}{\|\eigenvec_1\|_{\B}}$
        \STATE $\gparam_{\mathrm{opt}} = \hat{\gparam}_{\mathrm{opt}} + \B^{-1} \vecb$
        \RETURN $\gparam_{\mathrm{opt}}$
\end{algorithmic}
\end{algorithm}
For the inner computation of (\ref{kernel_eigen_prob}), we can use some algorithms such as the Arnoldi methods (see e.g. \citep{lehoucq1998arpack}) to find the eigenvectors of smallest generalized eigenvalue problems.
The computational complexity of \cref{kernel_bodame_algo} is $\mathcal{O}(\ntrain^{3})$, where $\ntrain$ denotes the dimension of $\gparam_{\mathrm{opt}}$.
 

%% file: contents/sgd_bodame.tex
\section{BODAME--SGD/SGA}
\label{sec:sgd-bodame}
While an attacker can find an analytic solution for the lower level optimization (\ref{inner_minimization}) in BODAME in the previous setting, it may be difficult to do so in more general cases. 
In an application, an attacker often learns a defender's model by using a few empirical samples by stochastic gradient descent (SGD) that stops in a few steps.
In this section, we assume that we can explicitly obtain the attacker's gradient.
We give an algorithm based on gradient ascent using such knowledge on the SGD attacker.
Compared with BODAME--KRR/KR, the algorithm can be applied to a wider range of defender models since it requires only differentiability of the models of the attacker and defender.

\subsection{Formulation assuming SGD attacker}
\label{subsec:sgd-formulation}

Recent machine learning models are trained with (mini--batch) SGD--based algorithms (with momentum) \citep{robbins1951stochastic, duchi2011adaptive, KingmaB2014adam}. Suppose that $\hparam^{(\maxstepinner)}(\gparam)$ is a parameter learned by an SGD--based algorithm in $\maxstepinner \in \mathbb{N}$ steps. More precisely, for $\stepinner = 0, \dots, \maxstepinner-1$, $\hparam^{(\stepinner+1)}$ is recurrently defined as 
\begin{align}
    \hparam^{(\stepinner+1)}(\gparam) = \hparam^{(\stepinner)}(\gparam) - \lrinner_{\stepinner+1} \nabla_{\hparam} \frac{1}{|\minibatchattacker_{\stepinner}|} \sum_{\idxattacker \in \minibatchattacker_{\stepinner}} \lattacker(\g_{\gparam}(\xattackersample), \h_{\hparam^{(\stepinner)}}(\xattackersample)),
\end{align}
where $\lrinner_{\stepinner+1} > 0$ is the learning rate, $\minibatchattacker_{\stepinner}$ means mini--batch data indices of $\{\xattackersample\}$, and $|\minibatchattacker_{\stepinner}|$ means the number of elements of $\minibatchattacker_{\stepinner}$, and $\hparam_{0} \in \hParam$, $\hparam^{(0)}(\gparam) = \hparam_{0}$ is given as an initial value. Here we use the same notation as in \cref{subsec:empirical-batch-bodame}. We summarize these update formulae below:
\begin{align}
    \hparam^{(\stepinner+1)} = \update_{\stepinner} (\hparam^{(\stepinner)}, \gparam),
\end{align}
where $\update_{\stepinner}: \hParam \times \gParam \to \hParam$. We assume that an attacker generates sequences of $\hparam$ by using update formulae $\update_{\stepinner}$ from some initial point since the attacker does not have analytic solutions for the lower--level optimization problem (\ref{expected-inner-minimization}).
Furthermore, we assume that the defender knows the attacker's settings as follows:
\begin{remark}[Defender's knowledge about a SGD attacker]
A defender knows $\update_{\stepinner}$ ($\stepinner = 0, \dots, \maxstepinner-1$), $\hparam_{0}$ and attacker's samples $\{\xattackersample\}_{\idxattacker=1}^{\nattacker}$. \label{sgd-assumption} 
\end{remark}
\cref{sgd-assumption} means that the defender knows $\hparam_{0}$, $\minibatchattacker_{\stepinner}$, $\lrinner_{\stepinner+1}$ $(\stepinner = 0, \dots, \maxstepinner-1)$, $\lattacker$, $\h_{\hparam}$ and $\hParam$, as well as the SGD--update formula in the SGD setting. Under \cref{sgd-assumption}, we transform the original expected BODAME (\ref{expected-objective}) (\ref{expected-inner-minimization}) (\ref{expected-constraint-quality}) as follows: 
\begin{align}
    \max_{\gparam \in \gParam} \quad &  \E_{\rvin \sim \Qdist} [\lobj(\f(\rvin), \h_{\hparam^{(\maxstepinner)}(\gparam)}(\rvin))], \label{sgd-1-objective}\\
    \mathrm{s.t.} \quad & \hparam^{(0)} = \hparam_{0}, \label{sgd-2-initial}\\
    \quad & \hparam^{(\stepinner+1)} = \update_{\stepinner} (\hparam^{(\stepinner)}, \gparam) \quad (\stepinner = 0 \dots, \maxstepinner - 1), \label{sgd-3-update}\\
    \quad &\E_{\rvin \sim \Qdist} [ \lconst(\f(\rvin), \g_{\gparam}(\rvin))] \le \epsquality. \label{sgd-4-constraint}
\end{align}
We can easily confirm that this formulation is equivalent to the original expected BODAME (\ref{expected-objective})--(\ref{expected-constraint-quality}) if $\lim_{\stepinner \to \infty} \hparam^{(\stepinner)} = \argmin_{\hparam \in \hParam} \E_{\rvin \sim \Pdist} [\lattacker(\g_{\gparam}(\rvin), \h_{\hparam}(\rvin))]$.

\subsection{Algorithm for BODAME--SGD/SGA}
We empirically optimize this problem by using an algorithm based on mini--batch Stochastic Gradient Ascent (SGA) w.r.t. $\gparam$.
We assume that $\gparam \in \mathbb{R}^{\dgParam}$ and $\hparam \in \mathbb{R}^{\dhParam}$, where $\dgParam, \dhParam \in \mathbb{N}$.
We introduce an objective loss function in each step, 
\begin{align}
    \Lobj_{\stepouter} (\gparam) = \frac{1}{|\minibatchobj_{\stepouter}|} \sum_{\idxobj \in \minibatchobj_{\stepouter}} \lobj (\f(\xobjsample), \h_{\hparam^{(\maxstepinner)}(\gparam)}(\xobjsample)),
\end{align}
where max iteration $\maxstepouter \in \mathbb{N}$ and steps $\stepouter =0, \dots, \maxstepouter - 1$, $\minibatchobj_{\stepouter}$ are indices of a subset of $\{\xobjsample\}_{\idxobj=1}^{\nobj}$.
We need to use the gradient of this loss function to apply a gradient ascent method.
Here, we assume that $\lobj, \lconst, \g, \h$ and $\update_{\stepinner} \quad (\stepinner = 0, \dots, \maxstepinner - 1)$ are differentiable.
In fact, the computation of $ \nabla_{\gparam} \Lobj_{\stepouter}(\gparam)$ is directly related to meta--learning or hyperparameter optimization, where gradients can be computed by using Hypergradient \citep{maclaurin2015gradient, franceschi2017forward, franceschi2018bilevel}.
We will use forward computation on $\stepinner$ to calculate Hypergradient.
We describe the batch constraint function corresponding to (\ref{sgd-4-constraint}) as $\Lconst(\gparam) \le \epsquality$, where
\begin{align}
    \Lconst (\gparam) = \frac{1}{\nconst} \sum_{\idxconst = 1}^{\nconst} \lconst (\f(\xconstsample), \g_{\gparam}(\xconstsample)).
\end{align}
To remove (\ref{sgd-4-constraint}), we use the log--barrier method (see e.g. \citep{forsgren2002interior}), adding a barrier function $\regconst \log(\epsquality-\Lconst(\gparam) )$ with the barrier parameter $\regconst$ to the objective function.
In so doing, we expect that, without having $\Lconst(\gparam) \le \epsquality$ as a constraint, this constraint must hold by having \cref{sgd-bodame-algo} start from a strictly feasible point $\gparam_{0}$ as a following assumption:
\begin{assumption}
    The initial point $\gparam_{0}$ is strictly feasible (i.e. $\Lconst(\gparam_{0}) < \epsquality$). \label{interior-assumption}
\end{assumption}
It is enough to set $\gparam_{0} = \gparam^{*}$, where $\gparam^{*}$ denotes the true parameter of $\f$ to satisfy \cref{interior-assumption} when $\fSpace = \gSpace$.
The reason for using full--batch constraint instead of a mini--batch one is that we need to guarantee that all $\gparam^{(\stepouter)}$ (a parameter learned in step $\stepouter$) are strictly feasible on each $\stepouter$ without randomness.
The optimization problem (\ref{sgd-1-objective})--(\ref{sgd-4-constraint}) based on a differentiable assumption of each model is called \textbf{BODAME--SGD/SGA}, as the attacker uses a model which can be optimized by SGD and the defender uses a model which can be optimized by SGA.
We finally obtain \cref{sgd-bodame-algo}.

\begin{algorithm}[ht]
\caption{Algorithm for BODAME--SGD/SGA with log--barrier}
\label{sgd-bodame-algo}
\begin{algorithmic}[1]
    \REQUIRE Initial values $\gparam_{0}$ and $\hparam_{0}$, max iteration $\maxstepouter \in \mathbb{N}$, learning rates $\lrouter_{\stepouter} > 0$ ($\stepouter =1, \dots, \maxstepouter$), mini--batch indices $\minibatchobj_{\stepouter}$ ($\stepouter=1, \dots, \maxstepouter$), a barrier parameter $\regconst$ and an approximation parameter $\epsquality$.
    \ENSURE Surrogate parameter $\gparam_{\maxstepouter}$
        \STATE \texttt{\# Initialization}
        \STATE $\gparam^{(0)} = \gparam_{0}$
        \STATE $\hparam^{(0)} = \hparam_{0}$
        \FOR{$\stepouter=0$ to $\maxstepouter-1$}
            \STATE \texttt{\# Calculate $\nabla_{\gparam} \Lobj_{\stepouter}(\gparam^{(\stepouter)})$ using Hypergradient}
            \STATE $\Z_{\stepouter}^{(0)} = O_{\dhParam \dgParam}$
            \FOR{$\stepinner = 0$ to $\maxstepinner-1$}
                \STATE $\Z_{\stepouter}^{(\stepinner+1)} = \frac{\partial \update_{\stepinner}(\hparam^{(\stepinner)}, \gparam^{(\stepouter)})}{\partial \hparam} \Z_{\stepouter}^{(\stepinner)} + \frac{\partial \update_{\stepinner}(\hparam^{(\stepinner)}, \gparam^{(\stepouter)})}{\partial \gparam}$
                \STATE $\hparam^{(\stepinner+1)} = \update_{\stepinner}(\hparam^{(\stepinner)}, \gparam^{(\stepouter)})$
            \ENDFOR
            \STATE $\nabla_{\gparam} \Lobj_{\stepouter}(\gparam^{(\stepouter)}) = (\Z_{\stepouter}^{(\maxstepinner)})^{\top} \nabla_{\hparam} \Lobj_{\stepouter}(\gparam^{(\stepouter)})$
            \STATE \texttt{\# Guarantee the next point to be strictly feasible using backtracking}
            \STATE $\Lconst (\gparam^{(\stepouter+1)}) = \infty$
            \WHILE{$\Lconst (\gparam^{(\stepouter+1)}) \ge \epsquality$}
                \STATE $\gparam^{(\stepouter+1)} = \gparam^{(\stepouter)} + \lrouter_{\stepouter+1} (\nabla_{\gparam} \Lobj_{\stepouter}(\gparam^{(\stepouter)}) + \regconst \nabla_{\gparam} \log (\epsquality -  \Lconst (\gparam^{(\stepouter)})))$
                \STATE $\lrouter_{\stepouter+1}$ = $\lrouter_{\stepouter+1}/2$
            \ENDWHILE
            \STATE (optional: $\regconst \to 0$)
        \ENDFOR
        \RETURN $\gparam_{\maxstepouter}$
\end{algorithmic}
\end{algorithm}

%% file: contents/kernel_experiments.tex
\section{Experiments on BODAME--KRR/KR}
\label{sec:exp-bodame-krr-kr}
We performed experiments on BODAME--KRR/KR by changing the difference between the attacker's and defender's distribution to see if a large difference would provide enough room for an effective defense in our BODAME framework.
We also evaluated the generalization ability of our surrogate model against new queries from an attacker.

In what follow, Let $\f_{\gparam^{*}}$ denotes the true model and $\g_{\gparam_{\mathrm{opt}}}$ denotes the optimized surrogate model.
Furthermore, $\h_{\tilde{\hparam}(\gparam_{\mathrm{opt}})}$ denotes the attacker's model learned by using the outputs of the surrogate model $\g_{\gparam_{\mathrm{opt}}}$ as supervised data and $\h_{\tilde{\hparam}(\gparam^{*})}$ denotes the one learned by using the outputs of the true model $\f_{\gparam^{*}}$. 
That is, $\h_{\tilde{\hparam}(\gparam_{\mathrm{opt}})}$ corresponds to the attacker's model built in our BODAME framework and  $\h_{\tilde{\hparam}(\gparam^{*})}$ corresponds to undefended naive deployment. These settings are illustrated in \cref{fig:overview-exp}.

\begin{figure}[ht]
    \centering
    \includegraphics[width=10cm]{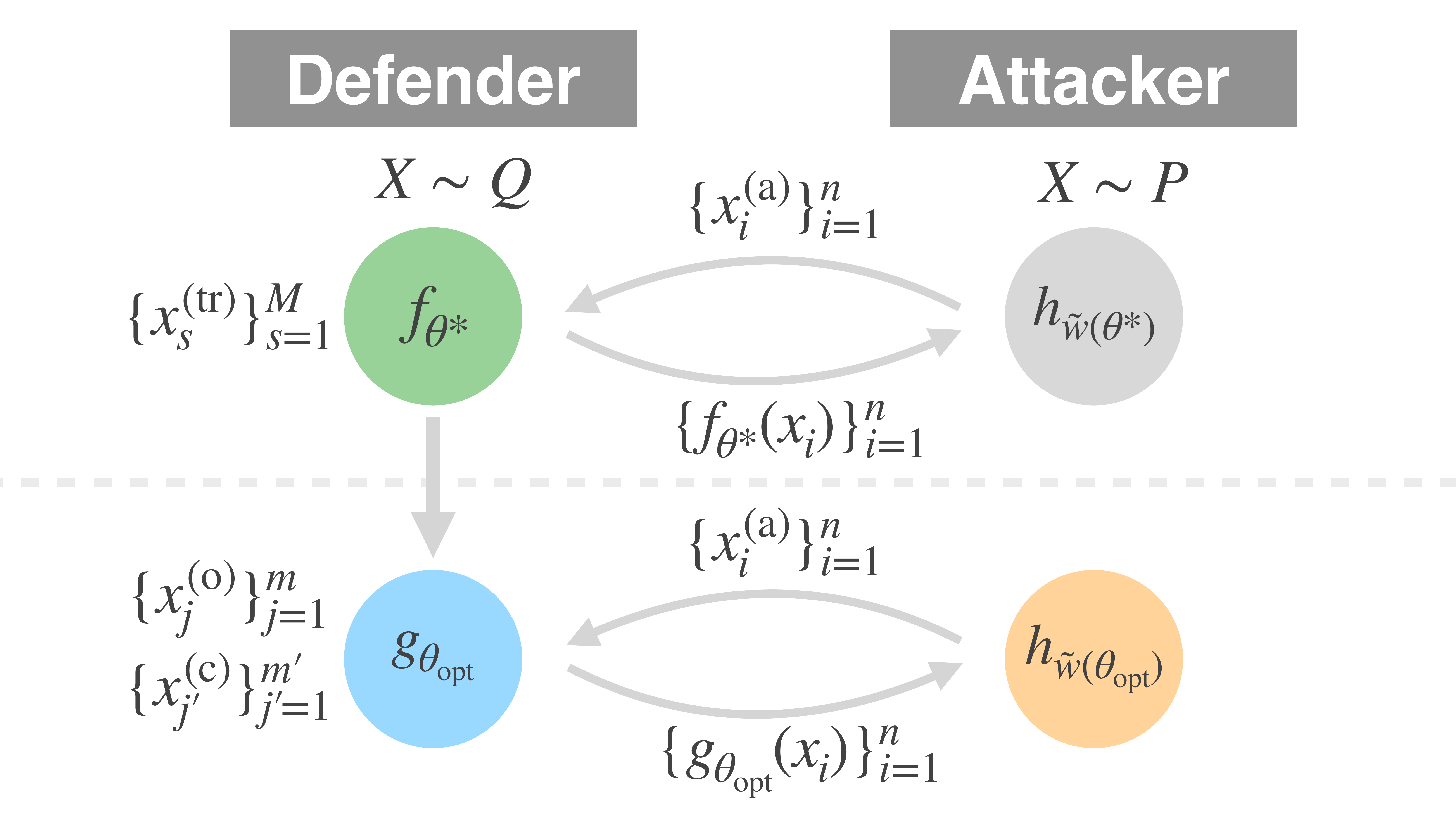}
    \caption{Attacker–defender framework in experiments.}
    \label{fig:overview-exp}
\end{figure}

We used white--wine data in the UCI wine dataset \citep{Dua:2019}.
The sample size was 4899, and the number of features was 11, excluding quality of wine.
We constructed models to predict the quality (3 -- 9) of wine in regression settings.
The experiments were conducted on an Intel Core i9 at 2.4 GHz.

\subsection{Effect of difference between distributions}
\label{subsec:dif-dist-kernel-exp}
We divided the wine dataset into defender's data, attacker's data, and test data.
To make the difference of two distributions, we shifted attacker's samples by adding noises that are normally distributed in $\mathcal{N}(\distmu \bm{1}_{11},  0.2^2 I_{11})$.
We set the parameters to $\ntrain = 350, \nattacker = 300, \nobj = 1000, \nconst = 1500, \epsquality = 0.1$, and $\lambdaattacker = 1.0$, and the test sample size is 500.
Both the defender and the attacker had RBF kernels, expressed as $\exp\{- 0.005 \|\queryin - \queryin'\|^{2}_{2} \}$ for $\queryin, \queryin' \in \Spacein$.
We compared our method with rounding a 16-decimal-place number to an integer in terms of disagreement and regression shown as \cref{tab:defense-methods}.
Here, $\g$ (Rounding) denotes the surrogate model using rounding of the output of the true model, and $\h$ (Rounding) denotes the attacker's model using $\g$ (Rounding) as supervised data.
We evaluated the mean squared error (MSE) of the defender's and attacker's models in test data to examine the effect of the difference between the distributions $\Qdist$ and $\Pdist$, since \cref{thm:trivial-case} and \cref{thm:infty-case} suggest the importance of the difference of two supports.

The results are in \cref{kernel-mse}.
We performed 50 experiments on each $\distmu$ by shuffling the data to preserve the size of each sample and aggregating the results to calculate the median.
The large value of $\h_{\tilde{\hparam}(\gparam_{\mathrm{opt}})}$ indicates that our defense method performed well especially in the case of a large $\distmu$ when the distributional gap between the attacker and defender was large.
Whereas undefended deployment, expressed as $\h_{\tilde{\hparam}(\gparam^{*})}$, was susceptible to the attacker's imitation even for large $\distmu$, the surrogate model made an adequate defense for large $\distmu$.
We also found that the prior defense by rounding had almost the same effect as undefended models.
We also observed that the MSE of $\g_{\gparam_{\mathrm{opt}}}$ stayed near to the MSE of the true model $\f_{\gparam^{*}}$, which implies that the constraint of BODAME--KRR/KR performed well.

\begin{figure}[ht]
    \centering
    \includegraphics[width=10cm]{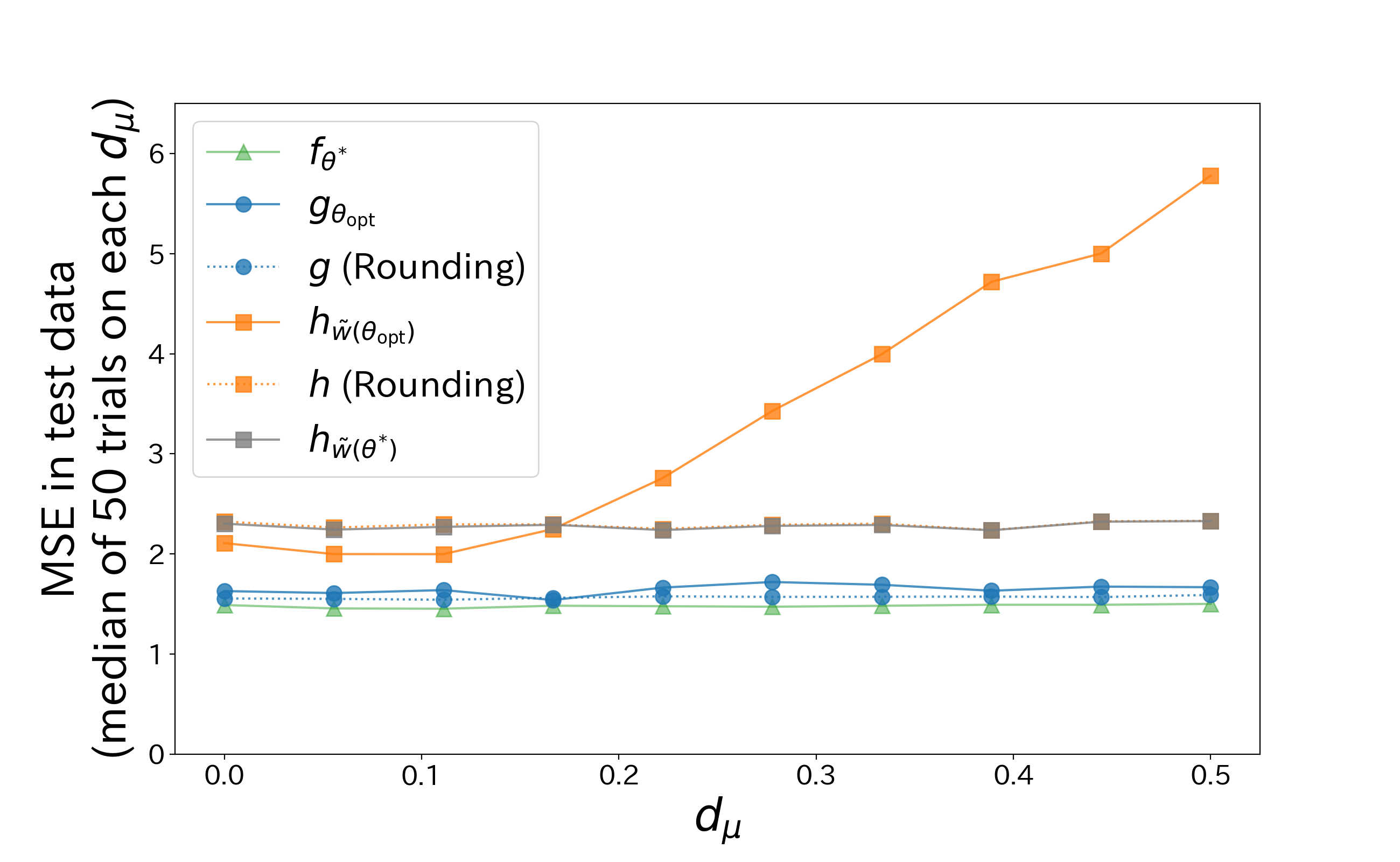}
    \caption{Relation between each model's MSE and the distance of the attacker's data from the defender's data in the UCI wine dataset. Dotted lines are for models in relation to the rounding defense.}
    \label{kernel-mse}
\end{figure}

\subsection{Generalization ability of surrogate model against noisy attacker's queries}
\label{subsec:generalization-kernel-experiments}
We tested the generalization ability of our surrogate model against attacker's queries in the BODAME--KRR/KR.
We optimized BODAME--KRR/KR and obtained a surrogate model $\g_{\gparam_{\mathrm{opt}}}$ with an MSE of 1.477 for the true model $\ftrue$ with an MSE of 1.377 on 500 test samples when the attacker's distribution was shifted by $\distmu = 0.5$.
Next, we trained the attacker's model using the same hyperparameters, but used new 300 queries that were generated by adding new noise that followed the same normal distribution for which $\distmu$ was 0.5 to attacker's data, by using supervised data by the fixed surrogate model $\g_{\gparam_{\mathrm{opt}}}$.
We performed 50 experiments in which we changed the normally distributed noise and computed the mean and the standard deviation of MSEs.

The results in  \cref{fig:transferability-data-kernel} imply that our surrogate models have generalization ability against the attacker's queries, since the attacker's models that were trained on new queries had almost the same MSEs as the attacker's model trained with the original queries.
This experiment suggests that the defender can prevent an attacker whose queries follow the same distribution used in BODAME--KRR/KR from extracting the defender's model solving BODAME--KRR/KR only once when the defender knows the attacker's kernel model as \cref{kernel-attacker-assumption} except for the attacker's samples.
The generalization ability of the surrogate model against attacker's queries is significantly different from rounding because it requires all queries on the attacker's original data due to the online property of each attack.

\begin{figure}[ht]
    \centering
    \includegraphics[width=10cm]{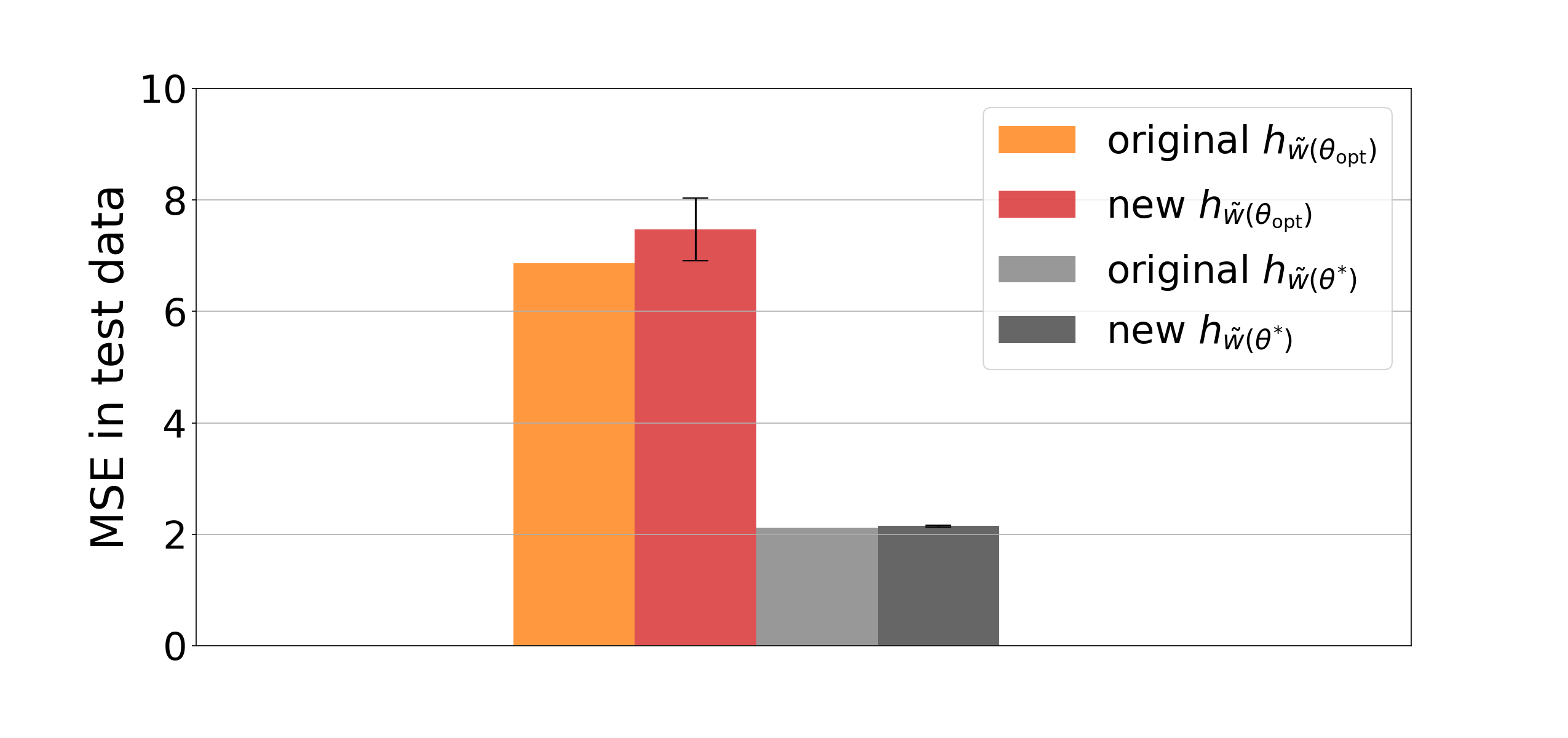}
    \caption{Results of experiments testing the generalization ability of surrogate models against attacker's queries in BODAME--KRR/KR. New models mean the attacker's models learned by using new attacker's queries. The error bar indicates the standard deviation.}
    \label{fig:transferability-data-kernel}
\end{figure}

%% file: contents/sgd_experiments.tex
\section{Experiments on BODAME--SGD/SGA}
\label{sec:exp-bodame-sgd-sga}
We performed similar experiments on BODAME--SGD/SGA to those on BODAME--KRR/KR, but using a different dataset.
Here, we used the MNIST dataset \citep{lecun1998gradient}, which is a dataset of hand--written digits from 0 to 9 in 28 $\times$ 28 pixels.
The training sample size was 60000 and the test sample size was 10000.
The experiments were conducted on one GPU of an NVIDIA Tesla T4.
We use the same notation as in \cref{sec:exp-bodame-krr-kr}.

\subsection{Effect of difference between distributions}
\label{subsec:sgd-exp-dif-dist}

We fixed defender's labels $\{0, 1, 2\}$ and changed the attacker's labels into $\{0, 1, 2\}, \{1, 2, 3\}, \{7, 8, 9\}$ to make the two distributions different.
\cref{fig:2d-vis-mnist} is a 2D visualization of $\Qdist$ and $\Pdist$ made by t--SNE \citep{maaten2008visualizing}.
It shows that the restriction of labels makes a difference between $\Qdist$ and $\Pdist$.
We used the same CNN consisting of two convolutional layers, a maxpooling layer, and two fully--connected layers as the defender's model and attacker's model.
We trained the defender's true model with 29399 samples which had full labels in 30 epochs using mini--batch SGD without momentum to make $\f_{\gparam^{*}}$.
We also pre--trained the attacker's model with 12600 samples, which had full labels, in 5 epochs by using mini--batch SGD without momentum.
To make the surrogate model $\g_{\gparam_{\mathrm{opt}}}$, we used transfer learning in \cref{sgd-bodame-algo} from the initial point $\gparam^{*}$.
Only the parameters in the final layers are learned in the attacker's model and surrogate model.
A squared loss function was commonly used for $\lobj, \lconst$, and $\lattacker$. The same $\{\xconstsample\}_{\idxconst=1}^{\nconst}$, where $\nconst = 1957$, was used in all experiments .
We set one epoch of the attacker and 15 steps for a mini-batch with a size of 64, and set one epoch of the defender and 15 steps for a mini-batch with a size of 64.
The parameters were $\lrouter_{\stepouter} = 0.3$ (for all $\stepouter = 1, \dots, 15$), $\lrinner_{\stepinner} = 0.01$ (for all $\stepinner = 1, \dots, 15$), $\regconst = 0.1$ and $\epsquality = 1.0$.

We evaluated the models using 3147 test samples that were not used for pre--training or transfer learning. \cref{fig:sgd-acc-dif-dist} shows that our model defends well when the defender's distribution is different from the attacker's distribution, since the difference in accuracy between $\hundefended$ and $\hdefended$ was about 20 \% in the setting in which the attacker's labels were $\{7, 8, 9\}$.
These results also suggest that the difference of distributions is an important factor in a defense against model extraction in BODAME--SGD/SGA.
We also observed that the accuracy of $\g_{\gparam_{\mathrm{opt}}}$ stayed near to the accuracy of the true model $\f_{\gparam^{*}}$, which implies that the constraint of BODAME--SGD/SGA performed well.

\begin{figure}[ht]
    \centering
    \includegraphics[width=13cm]{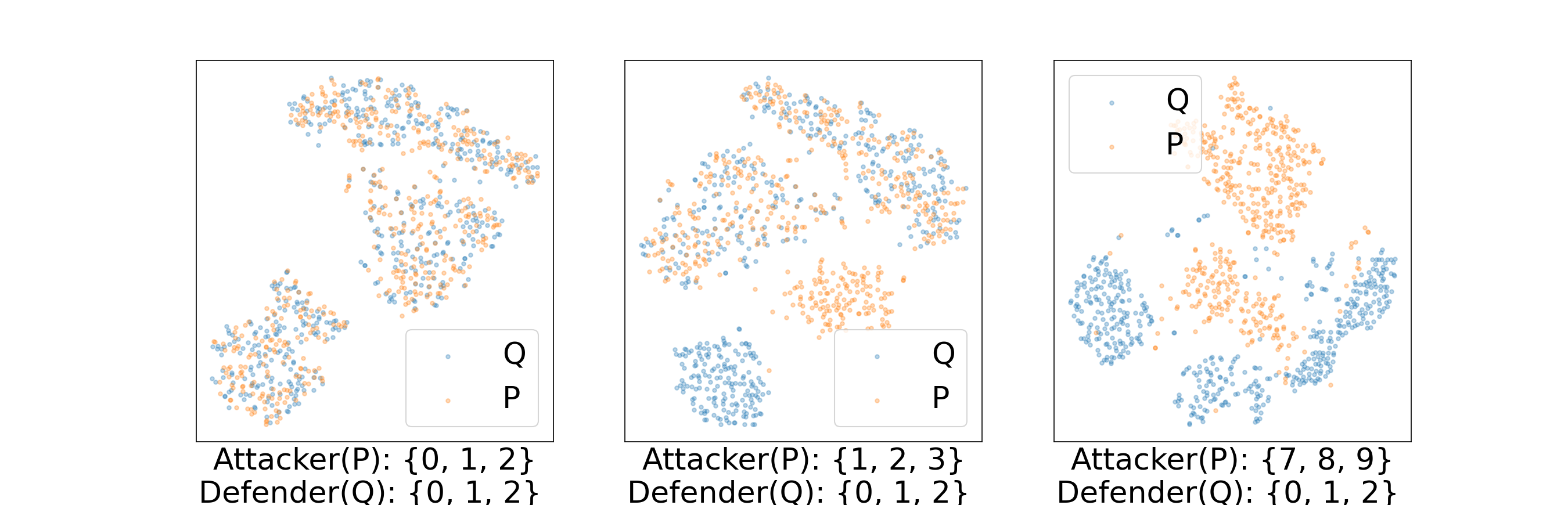}
    \caption{2D visualization of attacker's data and defender's data made by t--SNE conditioned on partial labels. We selected 500 samples from each $\Qdist$ and $\Pdist$ in each setting of partial labels.}
    \label{fig:2d-vis-mnist}
    \centering
    \includegraphics[width=10cm]{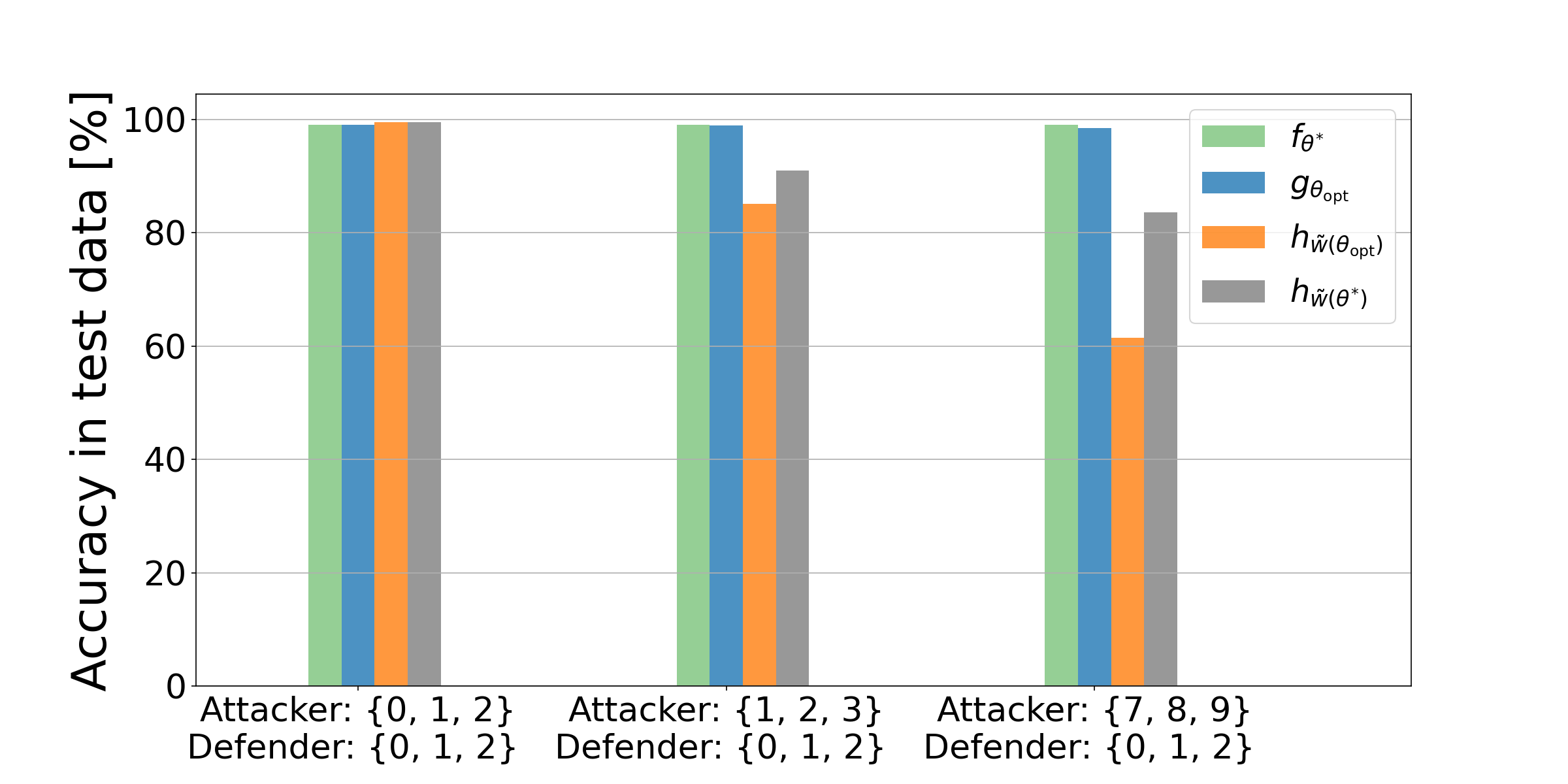}
    \caption{Relation between each model's accuracy and the difference of the attacker's data from defender's data in the MNIST dataset.}
    \label{fig:sgd-acc-dif-dist}
\end{figure}

\subsection{Generalization ability of surrogate model against noisy attacker's queries}
We tested the generalization ability of a surrogate model against attacker's queries as \cref{subsec:generalization-kernel-experiments} when $\Qdist$ and $\Pdist$ were respectively conditioned on labels $\{0, 1, 2\}$ and $\{7, 8, 9\}$ so that the difference between the distributions would be large.
We obtained a surrogate model $\g_{\gparam_{\mathrm{opt}}}$ with an accuracy of 98.47 for the true model $\ftrue$ with an accuracy of 99.05 by solving BODAME--SGD/SGA once.
Next, we trained the attacker's model by using new 768 queries with the same attacker's hyperparameters.
We performed ten experiments in which we randomly chose mini--batch data and computed the mean and the standard deviation of accuracy.

\begin{figure}[ht]
    \centering
    \includegraphics[width=10cm]{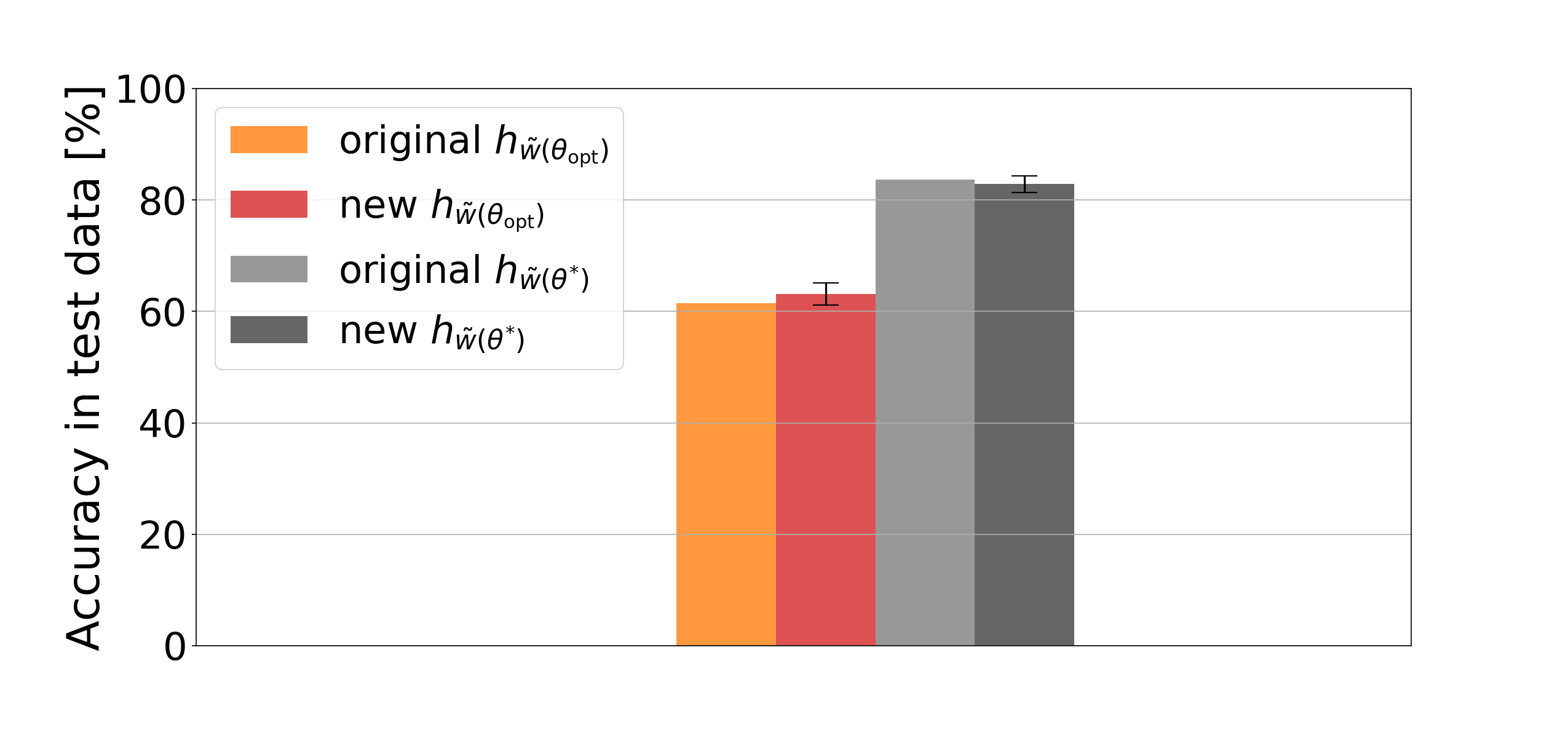}
    \caption{Results of experiments to confirm generalization ability of surrogate models on attacker's queries. New models mean the attacker models learned by using new attacker's queries. The error bar indicates the standard deviation.}
    \label{fig:transferability_data_sgd}
\end{figure}

The results are in \cref{fig:transferability_data_sgd}.
The surrogate model showed generalization ability against attacker's queries since the attacker's models that were trained on the new queries had almost the same accuracy as the original model.
This experiment also implies that the defender can prevent an attacker who has queries following the same distribution from extracting the defender's model by solving BODAME--SGD/SGA only once, which it is significantly different from the previous online defenses.

%% file: contents/conclusion.tex
\section{Conclusion}
\label{sec:conclusion}
We formulated and theoretically analyzed a defense against model extraction framed as a bilevel optimization problem.
In addition, under the assumption of kernel models or SGD attacker, we showed respectively the optimization problem and the algorithms to find the solutions.
In numerical experiments, we showed that the proposed defense better mitigates the attacker's model extraction attack compared with the case without defense and with a prior defense.
We also confirmed the generalization ability of the defender's surrogate model against attacker's queries in numerical experiments.
The results suggests that the proposed defense has an advantage with the prior online defense against individual queries in that the same model can be used for the defense once our optimization problem is solved.

%% file: contents/acknowledgements.tex
\section*{Acknowledgments}
AN was partially supported by JSPS Kakenhi (19K20337) and JST-PRESTO.

%% file: contents/appendix.tex
\part*{\Large{Appendix}}

\section{Notation}
\noindent
\textbf{Models.} A defender has a true model as $\f \in \fSpace: \Spacein \to \Spaceout$, where $\fSpace$ is a function space, $\Spacein$ is an input space and $\Spaceout$ is an output space.
A defender builds a parameterized surrogate model $\g_{\gparam} \in \gSpace: \Spacein \to \Spaceout$, where $\gSpace$ is a function space, $\gparam \in \gParam$, and $\gParam$ is a parameter space.
If $\f$ is parameterized by $\gparam^{*} \in \gParam$, it is denoted by $\f_{\gparam^{*}}$.
An attacker builds a parameterized model $\h_{\hparam} \in \hSpace$ by model extraction where $\hSpace$ is a function space, $\hparam \in \hParam$, and $\hParam$ is a parameter space.
$\g_{\gparam}$ is simply written by $\g$ and $\h_{\hparam}$ is also simply written by $\h$ when we discuss functionally properties.
If $\gParam$ is a Euclidean space, $\gParam$ indicates $\mathbb{R}^{\dgParam}$, where $\dgParam \in \mathbb{N}$.
If $\hParam$ is a Euclidean space, $\hParam$ indicates $\mathbb{R}^{\dhParam}$, where $\dhParam \in \mathbb{N}$.
When we assumue that a true model uses a positive definite kernel, the kernel is denoted by $\ktrue$.
When we assume that an attacker's model uses a positive definite kernel, the kernel is denoted by $\kattacker$. 

\noindent
\textbf{Distribution and samples.}
A defender has input distribution  $\Qdist$ on $\Spacein$. An attacker has input distribution $\Pdist$ on $\Spacein$. In empirical settings, A defender has objective samples $\{\xobjsample\}_{\idxobj=1}^{\nobj}$ ($\xobjsample \sim \Qdist$ i.i.d., $\nobj \in \mathbb{N}$) and constraint samples $\{\xconstsample\}_{\idxconst=1}^{\nconst}$ ($\xconstsample \sim \Qdist$ i.i.d., $\nconst \in \mathbb{N}$). In empirical settings, an attacker has samples $\{\xattackersample\}_{\idxattacker=1}^{\nattacker}$ ($\xattackersample \sim \Pdist$, i.i.d., $\nattacker \in \mathbb{N})$. Training samples to learn a true model $\f$ are denoted by $\{\xtrainsample\}_{\idxtrain=1}^{\ntrain}$ ($\xtrainsample \sim \Qdist$, i.i.d., $\ntrain \in \mathbb{N}$).

\noindent
\textbf{Loss functions.}
A defender has an objective loss function $\lobj : \Spaceout \times \Spaceout \to \mathbb{R}$ and constraint loss function $\lconst : \Spaceout \times \Spaceout \to \mathbb{R}$. An attacker has a loss function $\lattacker: \Spaceout \times \Spaceout \to \mathbb{R}$.

\noindent
\textbf{Surrogate quality.} $\epsquality > 0$ is a parameter to control quality of a surrogate model $\g$ to approximate $\f$.

\noindent
\textbf{Hyperparameters.} $\lambdaattacker > 0$ is a coefficient of regularization in $\lattacker$.
$\lrinner$ is a learning rate of an attacker that uses SGD.
$\lrouter$ is a learning rate used in defender's optimization for defense. $\regconst$ is a barrier parameter used in \cref{sgd-bodame-algo}.

\noindent
\textbf{Symbols.}
Support of distribution is denoted by $\supp(\cdot)$. Indicator function is denoted by $\ind$. Kernel of linear map $A$ is denoted by $\Ker(A)$. Sign function is denoted by $\mathrm{sgn}: \mathbb{R} \to \{-1, 1\}$, 
\begin{align}
    \mathrm{sgn}(x) = \begin{cases} 1 & (x > 0)\\ -1 & (x \le 0)\end{cases} .
\end{align}

\section{Proofs}

\subsection{Proofs of Theorem 1 and Theorem 2}
We will denote the expected BODAME as a function approximation as follows:
\begin{align}
    \max_{\g \in \gSpace} \quad & \E_{\rvin \sim \Qdist} [\lobj(\f(\rvin), \h_{\g}(\rvin))], \label{approx-expected-objective} \\
    \mathrm{s.t.} \quad & \h_{\g} = \argmin_{\h \in \hSpace} \E_{\rvin \sim \Pdist} [\lattacker(\g(\rvin), \h(\rvin))], \label{approx-expected-inner-minimization}  \\
    \quad &  \E_{\rvin \sim \Qdist} [\lconst(\f(\rvin), \g(\rvin)) ]  \le \epsquality. \label{approx-expected-constraint-quality}
\end{align}

We give a proof of \cref{thm:trivial-case}.
\begin{proof}[Proof of  \cref{thm:trivial-case}]
Minimum of $\E_{\rvin \sim \Pdist} [\lattacker(\g(\rvin), \h(\rvin))]$ is obviously $0$ since $\lattacker \ge 0$, $\gSpace = \hSpace$ and we can set $\h(\queryin) = \g(\queryin)$ for all $\queryin \in \Spacein$. Letting $\ind$ be an indicator function, we have
\begin{align}
    \E_{\rvin \sim \Pdist} [\lattacker(\g(\rvin), \h(\rvin))] &= \E_{\rvin \sim \Pdist} [\ind_{\{\g(\rvin) \ne \h(\rvin) \}}\lattacker(\g(\rvin), \h(\rvin))] + \E_{\rvin \sim \Pdist} [\ind_{\{\g(\rvin) = \h(\rvin) \}}\lattacker(\g(\rvin), \h(\rvin))] \\
    &= \E_{\rvin \sim \Pdist} [\ind_{\{\g(\rvin) \ne \h(\rvin) \}}\lattacker(\g(\rvin), \h(\rvin))].
\end{align}
Since $\lattacker(\queryout, \queryout') > 0$ if $\queryout \ne \queryout'$ and minimum of $\E_{\rvin \sim \Pdist} [\lattacker(\g(\rvin), \h(\rvin))]$ is 0, it is necessary that all $\h_{\g}$ satisfies $\h_{\g}(\rvin) = \g(\rvin)$ a.s. ($\rvin \sim \Pdist$). This condition is equivalent to $\h_{\g}(\rvin) = \g(\rvin)$ a.s. ($\rvin \sim \Qdist$) from the assumption $\supp(\Qdist) = \supp(\Pdist)$. Therefore, $\E_{\rvin \sim \Qdist} [\lobj(\f(\rvin), \h_{\g}(\rvin))] = \E_{\rvin \sim \Qdist} [\lconst(\f(\rvin), \g(\rvin)) ]$. As a result, $\E_{\rvin \sim \Qdist} [\lobj(\f(\rvin), \h_{\g}(\rvin))]$ can increase to $\epsquality$ from the constraint $\E_{\rvin \sim \Qdist} [\lconst(\f(\rvin), \g(\rvin)) ]  \le \epsquality$.
\end{proof}

We also give a proof of \cref{thm:infty-case}.
\begin{proof}[Proof of  \cref{thm:infty-case}]
For $\C > 0$, suppose that
\begin{align}
    \g(\queryin) = \begin{cases}
        \f(x) + \C & (\queryin \in \supp(\Pdist))\\
        \f(x) & (\mathrm{otherwise}).
    \end{cases}
\end{align}
This $\g$ satisfies $\E_{\rvin \sim \Qdist} [\lconst(\f(\rvin), \g(\rvin))] = 0$. $\h_{\g}(\rvin) = \g(\rvin)$ a.s. ($\rvin \sim \Pdist$) holds from the assumption of $\lattacker$ as shown in the proof of \cref{thm:trivial-case}. It should be noted that when we take $\f^{*} \in \fSpace$ for $\f$ and $\h_{\g}^{*} \in \hSpace$ for $\h_{\g}$ which satisfy for $\queryin \in [b, b + \C / 2 \lipconstant]$,
\begin{align}
    \f^{*} (\queryin) &= \lipconstant \queryin - b \lipconstant + \f(b),\\
    \h_{\g}^{*} (\queryin) &= - \lipconstant \queryin + b \lipconstant + \f(b) + \C,
\end{align}
and $\f^{*} (\rvin) = \h_{\g}^{*}(\rvin)$ a.s. ($\rvin \sim \Qdist$) on $\supp(\Qdist) \setminus [b, b + \C / 2 \lipconstant]$, the following inequality holds:
\begin{align}
    \E_{\rvin \sim \Qdist} [\lobj(\f(\rvin), \h_{\g}(\rvin))] \ge \E_{\rvin \sim \Qdist} [\lobj(\f^{*}(\rvin), \h_{\g}^{*}(\rvin))] \label{ineq:thm2-obj}
\end{align}
from the assumption that $\fSpace$ and $\hSpace$ are $\lipconstant$-Lipschitz function spaces. For the right--hand side in (\ref{ineq:thm2-obj}), we have the following, 
\begin{align}
    \E_{\rvin \sim \Qdist} [\lobj(\f^{*}(\rvin), \h_{\g}^{*}(\rvin))]
    &= \int_{c}^{d} (\f^{*}(\queryin) -  \h_{\g}^{*}(\queryin))^{2} d\Qdist (\queryin) \\
    &= \int_{c}^{d} (\C + 2 b \lipconstant - 2 \lipconstant \queryin)^2 d\Qdist (\queryin)\\
    &\to \infty \quad (\C \to \infty),
\end{align}
and therefore, the claim of \cref{thm:infty-case} holds.
\end{proof}

\subsection{Proof of Theorem \ref{global_opt_kernel}}
We use the following Lemma (see e.g. Theorem 1.1, \citep{horst1996global}) mentioned in our paper:
\begin{lemma}[Concave function minimization on a nonempty compact convex set]
Let $\Set \subset \mathbb{R}^\ntrain$ be nonempty, compact and convex. Let $L: \Set \to \mathbb{R}$ be concave. Then the global minimum of $L$ on $\Set$ is attained at an extreme point of $\Set$. \label{concave-min}
\end{lemma}

We give a proof of Theorem \ref{global_opt_kernel}.
\begin{proof}[Proof of Theorem \ref{global_opt_kernel}.]
Since $\A$ is positive semi--definite, $\hat{\A} \preceq O$ holds, and therefore, the objective of (\ref{acutual-problem}) is concave.
$\B$ is positive definite from the \cref{B-PDS-assumption}, $\{\hat{\gparam}  | \|\hat{\gparam}\|_{\B} \le \hat{\epsquality} \}$ is closed and bounded.
Furthermore, $\{\hat{\gparam} | \|\hat{\gparam}\|_{\B} \le \hat{\epsquality} \}$ contains a feasible solution $\gparam^{*} - \B^{-1} \vecb$, that implies that the set is nonempty.
Using \cref{concave-min}, we can guarantee that the problem (\ref{acutual-problem}) attains its minimum on $\{\hat{\gparam}  | \|\hat{\gparam}\|_{\B} = \hat{\epsquality} \}$. From Theorem 3.2 and Theorem 3.3 in \citep{adachi2017solving}, we obtain the global optimum of (\ref{acutual-problem}) as $\hat{\gparam}_{\mathrm{opt}} = - \mathrm{sgn} (\hat{\veca}^{\top} \eigenvec_2) \hat{\epsquality} \eigenvec_1 / \|\eigenvec_1\|_{\B}$ in $\{\hat{\gparam} | \|\hat{\gparam}\|_{\B} = \hat{\epsquality} \}$ where $\eigenvec_1, \eigenvec_2$ satisfies for the smallest eigenvalue $\eigenvalue$,
\begin{align}
            \begin{pmatrix} - \B & \hat{\A} \\ \hat{\A} & - \frac{\hat{\veca}\hat{\veca}^{\top}}{\hat{\epsquality}^2}\end{pmatrix} \begin{pmatrix} \eigenvec_1 \\ \eigenvec_2 \end{pmatrix} = \eigenvalue \begin{pmatrix} O & \B \\ \B & O \end{pmatrix} \begin{pmatrix} \eigenvec_1 \\ \eigenvec_2\end{pmatrix}
\end{align}
on the assumption $\hat{\veca} \not\perp \eigenspace$. By changing $\gparam$ to $\gparam - \B^{-1} \vecb$, we obtain the global optimum of the problem (\ref{kernel-objective}) and (\ref{kernel-constraint}).
\end{proof}

\section{Global optimization for BODAME--KRR/KR}
\label{appendix:kernel-bodame-hard}

In this section, we show an algorithm for the hard case of BODAME--KRR/KR that $\hat{\veca} \perp \eigenspace$ holds. We use the same notation in \cref{sec:kernel-bodame}.

\begin{theorem}[Global optimization for BODAME--KRR/KR]
    \cref{kernel_bodame_algo_hard} for the case of $\hat{\veca} \perp \eigenspace$ and \cref{kernel_bodame_algo} for the other cases output a a global optimum solution of BODAME--KRR/KR (\ref{kernel-objective}) and (\ref{kernel-constraint}). \label{thm:global-opt-hard}
\end{theorem}

\begin{proof}[Proof of Theorem \ref{thm:global-opt-hard}]
If $\hat{\veca} \not\perp \eigenspace$ holds, we can get global optimum of (\ref{kernel-objective}) and (\ref{kernel-constraint}) using \cref{global_opt_kernel}.
If  $\hat{\veca} \perp \eigenspace$ holds, we can get a global optimum of (\ref{kernel-objective}) and (\ref{kernel-constraint}) using Theorem 4.3 in \citep{adachi2017solving}.
We can distinguish the hard case that $\hat{\veca} \perp \eigenspace$ holds by checking whether the $\B$--norm of eigenvalues $\|\eigenvec_{1}\|_{\B}$ equals to 0.
Therefore, we first check the condition that $\|\eigenvec_{1}\|_{\B}$ equals to 0 by following \cref{global_opt_kernel} and then, if necessary, run \cref{kernel_bodame_algo_hard} to find the global optimum of BODAME--KRR/KR (\ref{kernel-objective}) and (\ref{kernel-constraint}).
\end{proof}

\begin{algorithm}
\caption{Algorithm for BODAME--KRR/KR where $\hat{\veca} \perp \eigenspace$ holds}
\label{kernel_bodame_algo_hard}
\begin{algorithmic}[1]
    \REQUIRE $\A$, $\veca$, $\B$, $\epsquality$, $\vecb$, $\eigenvalue$, $\gammab$
    \ENSURE Surrogate parameter $\gparam_{\mathrm{opt}}$
        \STATE Compute $\B^{-1} \vecb, \hat{\A}, \hat{\veca}, \hat{\epsquality}$
        \STATE Compute the $\B$--orthogonal null vectors of $\hat{\A} + \eigenvalue \B$, $\nulleigenvec_{1}, \dots, \nulleigenvec_{\nulldim}$, where $\nulldim = \dim(\Ker(\hat{\A} + \eigenvalue \B))$
        \STATE Fix $\eigenkkt > 0$
        \STATE Compute $\Hardmat = \hat{\A} + \eigenvalue \B + \eigenkkt \sum_{i=1}^{\nulldim} \B \nulleigenvec_{i} \nulleigenvec_{i}^{\top} \B $
        \STATE Solve $\Hardmat q + \hat{\veca} = 0$ for $q$ by the Conjugate Gradient Method
        \STATE Take an eigenvector $\nulleigenvec \in \Ker(\hat{\A} + \eigenvalue \B)$ and find $\beta \in \mathbb{R}$ such that $\|q + \beta \nulleigenvec\|_{\B} = \hat{\epsquality}$
        \STATE $\gparam_{\mathrm{opt}} = \q + \beta \nulleigenvec + \B^{-1}b$
        \RETURN $\gparam_{\mathrm{opt}}$
\end{algorithmic}
\end{algorithm}

In practice, we may use adding a perturbation such as normal noise to $\hat{\veca}$ to avoid the hard case using \cref{kernel_bodame_algo_hard} that $\hat{\veca} \perp \eigenspace$ holds; we only use the \cref{kernel_bodame_algo} by the perturbation to solve BODAME--KRR/KR (\ref{kernel-objective}) and (\ref{kernel-constraint}) although slight change of $\hat{\veca}$ may not lead to a global optimum of the original problem.

%% file: main.bbl
\begin{thebibliography}{}

\bibitem[Adachi et~al., 2017]{adachi2017solving}
Adachi, S., Iwata, S., Nakatsukasa, Y., and Takeda, A. (2017).
\newblock Solving the trust-region subproblem by a generalized eigenvalue
  problem.
\newblock {\em SIAM Journal on Optimization}, 27(1):269--291.

\bibitem[Alabdulmohsin et~al., 2014]{alabdulmohsin2014adding}
Alabdulmohsin, I.~M., Gao, X., and Zhang, X. (2014).
\newblock Adding robustness to support vector machines against adversarial
  reverse engineering.
\newblock In {\em 23rd ACM International Conference on Conference on
  Information and Knowledge Management}, pages 231--240.

\bibitem[Bastani et~al., 2017]{bastani2017interpretability}
Bastani, O., Kim, C., and Bastani, H. (2017).
\newblock Interpretability via model extraction.
\newblock In {\em Workshop on Fairness, Accountability, and Transparency}.

\bibitem[Batina et~al., 2019]{batina2019csi}
Batina, L., Bhasin, S., Jap, D., and Picek, S. (2019).
\newblock {CSI} {NN}: Reverse engineering of neural network architectures
  through electromagnetic side channel.
\newblock In {\em 28th {USENIX} Security Symposium}, pages 515--532.

\bibitem[Chandrasekaran et~al., 2020]{chandrasekaran2020exploring}
Chandrasekaran, V., Chaudhuri, K., Giacomelli, I., Jha, S., and Yan, S. (2020).
\newblock Exploring connections between active learning and model extraction.
\newblock In {\em 29th {USENIX} Security Symposium}, pages 1309--1326.

\bibitem[Dua and Graff, 2017]{Dua:2019}
Dua, D. and Graff, C. (2017).
\newblock {UCI} machine learning repository.

\bibitem[Duchi et~al., 2011]{duchi2011adaptive}
Duchi, J., Hazan, E., and Singer, Y. (2011).
\newblock Adaptive subgradient methods for online learning and stochastic
  optimization.
\newblock {\em Journal of Machine Learning Research}, 12(7).

\bibitem[Forsgren et~al., 2002]{forsgren2002interior}
Forsgren, A., Gill, P.~E., and Wright, M.~H. (2002).
\newblock Interior methods for nonlinear optimization.
\newblock {\em SIAM review}, 44(4):525--597.

\bibitem[Franceschi et~al., 2017]{franceschi2017forward}
Franceschi, L., Donini, M., Frasconi, P., and Pontil, M. (2017).
\newblock Forward and reverse gradient-based hyperparameter optimization.
\newblock In {\em 34th International Conference on Machine Learning}, pages
  1165--1173.

\bibitem[Franceschi et~al., 2018]{franceschi2018bilevel}
Franceschi, L., Frasconi, P., Salzo, S., Grazzi, R., and Pontil, M. (2018).
\newblock Bilevel programming for hyperparameter optimization and
  meta-learning.
\newblock In {\em 35th International Conference on Machine Learning}, pages
  1568--1577.

\bibitem[Fredrikson et~al., 2015]{fredrikson2015model}
Fredrikson, M., Jha, S., and Ristenpart, T. (2015).
\newblock Model inversion attacks that exploit confidence information and basic
  countermeasures.
\newblock In {\em 22nd ACM SIGSAC Conference on Computer and Communications
  Security}, pages 1322--1333.

\bibitem[Goodfellow et~al., 2015]{goodfellow2015explaining}
Goodfellow, I., Shlens, J., and Szegedy, C. (2015).
\newblock Explaining and harnessing adversarial examples.
\newblock In {\em 3rd International Conference on Learning Representations}.

\bibitem[Horst and Tuy, 1996]{horst1996global}
Horst, R. and Tuy, H. (1996).
\newblock {\em Global optimization: Deterministic approaches}.
\newblock Springer Science \& Business Media, third edition.

\bibitem[Jagielski et~al., 2020]{jagielski2020high_accuracy}
Jagielski, M., Carlini, N., Berthelot, D., Kurakin, A., and Papernot, N.
  (2020).
\newblock High accuracy and high fidelity extraction of neural networks.
\newblock In {\em 29th {USENIX} Security Symposium}, pages 1345--1362.

\bibitem[Juuti et~al., 2019]{juuti2019prada}
Juuti, M., Szyller, S., Marchal, S., and Asokan, N. (2019).
\newblock {PRADA}: protecting against {DNN} model stealing attacks.
\newblock In {\em 4th IEEE European Symposium on Security and Privacy}, pages
  512--527.

\bibitem[Kesarwani et~al., 2018]{kesarwani2018model}
Kesarwani, M., Mukhoty, B., Arya, V., and Mehta, S. (2018).
\newblock Model extraction warning in mlaas paradigm.
\newblock In {\em 34th Annual Computer Security Applications Conference}, pages
  371--380.

\bibitem[Kingma and Ba, 2015]{KingmaB2014adam}
Kingma, D.~P. and Ba, J. (2015).
\newblock Adam: {A} method for stochastic optimization.
\newblock In {\em 3rd International Conference on Learning Representations}.

\bibitem[Krishna et~al., 2020]{krishna2020thieves}
Krishna, K., Tomar, G.~S., Parikh, A., Papernot, N., and Iyyer, M. (2020).
\newblock Thieves of sesame street: Model extraction on {BERT}-based {API}s.
\newblock In {\em 8th International Conference on Learning Representations}.

\bibitem[Lecun et~al., 1998]{lecun1998gradient}
Lecun, Y., Bottou, L., Bengio, Y., and Haffner, P. (1998).
\newblock Gradient-based learning applied to document recognition.
\newblock In {\em IEEE}, pages 2278--2324.

\bibitem[Lee et~al., 2019]{lee2019defending}
Lee, T., Edwards, B., Molloy, I., and Su, D. (2019).
\newblock Defending against neural network model stealing attacks using
  deceptive perturbations.
\newblock In {\em IEEE Security and Privacy Workshops}, pages 43--49.

\bibitem[Lehoucq et~al., 1998]{lehoucq1998arpack}
Lehoucq, R.~B., Sorensen, D.~C., and Yang, C. (1998).
\newblock {\em ARPACK users' guide: solution of large-scale eigenvalue problems
  with implicitly restarted Arnoldi methods}.
\newblock SIAM.

\bibitem[Lowd and Meek, 2005]{lowd2005adversarial}
Lowd, D. and Meek, C. (2005).
\newblock Adversarial learning.
\newblock In {\em 11th ACM SIGKDD International Conference on Knowledge
  Discovery in Data Mining}, pages 641--647.

\bibitem[Maaten and Hinton, 2008]{maaten2008visualizing}
Maaten, L. v.~d. and Hinton, G. (2008).
\newblock Visualizing data using t-{SNE}.
\newblock {\em Journal of Machine Learning Research}, 9(Nov):2579--2605.

\bibitem[Maclaurin et~al., 2015]{maclaurin2015gradient}
Maclaurin, D., Duvenaud, D., and Adams, R. (2015).
\newblock Gradient-based hyperparameter optimization through reversible
  learning.
\newblock In {\em 32nd International Conference on Machine Learning}, pages
  2113--2122.

\bibitem[Milli et~al., 2019]{milli2019model}
Milli, S., Schmidt, L., Dragan, A.~D., and Hardt, M. (2019).
\newblock Model reconstruction from model explanations.
\newblock In {\em 3rd Conference on Fairness, Accountability, and
  Transparency}, pages 1--9.

\bibitem[Oh et~al., 2018]{oh2018towards}
Oh, S.~J., Augustin, M., Fritz, M., and Schiele, B. (2018).
\newblock Towards reverse-engineering black-box neural networks.
\newblock In {\em 6th International Conference on Learning Representations}.

\bibitem[Orekondy et~al., 2019]{orekondy2019knockoff}
Orekondy, T., Schiele, B., and Fritz, M. (2019).
\newblock Knockoff nets: Stealing functionality of black-box models.
\newblock In {\em 32nd IEEE Conference on Computer Vision and Pattern
  Recognition}, pages 4954--4963.

\bibitem[Orekondy et~al., 2020]{orekondy2020prediction}
Orekondy, T., Schiele, B., and Fritz, M. (2020).
\newblock Prediction poisoning: Towards defenses against {DNN} model stealing
  attacks.
\newblock In {\em 8th International Conference on Learning Representations}.

\bibitem[Pal et~al., 2020]{pal2020activethief}
Pal, S., Gupta, Y., Shukla, A., Kanade, A., Shevade, S., and Ganapathy, V.
  (2020).
\newblock Activethief: Model extraction using active learning and unannotated
  public data.
\newblock In {\em 34th AAAI Conference on Artificial Intelligence}, pages
  865--872.

\bibitem[Robbins and Monro, 1951]{robbins1951stochastic}
Robbins, H. and Monro, S. (1951).
\newblock A stochastic approximation method.
\newblock {\em The annals of mathematical statistics}, pages 400--407.

\bibitem[Rolnick and Kording, 2020]{rolnick2020reverse}
Rolnick, D. and Kording, K. (2020).
\newblock Reverse-engineering deep relu networks.
\newblock In {\em 37th International Conference on Machine Learning}, pages
  8178--8187.

\bibitem[Sch{\"o}lkopf et~al., 2002]{scholkopf2002learning}
Sch{\"o}lkopf, B., Smola, A.~J., Bach, F., et~al. (2002).
\newblock {\em Learning with kernels: support vector machines, regularization,
  optimization, and beyond}.
\newblock MIT press.

\bibitem[Szegedy et~al., 2014]{szegedy2014intriguing}
Szegedy, C., Zaremba, W., Sutskever, I., Bruna, J., Erhan, D., Goodfellow, I.,
  and Fergus, R. (2014).
\newblock Intriguing properties of neural networks.
\newblock In {\em 2nd International Conference on Learning Representations}.

\bibitem[Tramèr et~al., 2016]{tramer2016stealing}
Tramèr, F., Zhang, F., Juels, A., Reiter, M.~K., and Ristenpart, T. (2016).
\newblock Stealing machine learning models via prediction {API}s.
\newblock In {\em 25th {USENIX} Security Symposium}, pages 601--618.

\bibitem[Zheng et~al., 2019]{zheng2019bdpl}
Zheng, H., Ye, Q., Hu, H., Fang, C., and Shi, J. (2019).
\newblock {BDPL}: A boundary differentially private layer against machine
  learning model extraction attacks.
\newblock In {\em 24th European Symposium on Research in Computer Security},
  pages 66--83.

\end{thebibliography}
